\documentclass[journal]{IEEEtran}
\usepackage{diagbox}
\usepackage{booktabs}
\usepackage{multirow}
\usepackage{graphicx}
\usepackage{subfigure}
\usepackage{amsmath, amsthm, amssymb}

\usepackage{color}
\usepackage{times}
\usepackage{subfig}
\usepackage{cite}
\usepackage{stfloats}
\usepackage{float}
\usepackage{placeins}
\allowdisplaybreaks
\usepackage{array}
\usepackage{xcolor,graphicx}
\definecolor{darkgreen}{HTML}{006400} 

\usepackage{xcolor, soul}
\sethlcolor{green}

\newtheorem{remark}{Remark}

\newtheorem{theorem}{Theorem}

\newtheorem{assumption}{Assumption}
\newtheorem{lemma}{Lemma}

\begin{document}
%
\title{Stability and Comfort in Mobile Robot-Pedestrian Interactions}
%
%
%

\author{\textcolor{black}{Alireza Jafari, Hong-Son Nguyen, and Yen-Chen Liu$^*$}
\thanks{\textcolor{black}{This work was supported in part by the National Science and Technology Council (NSTC), Taiwan, under Grant NSTC 114-2628-E-006-010 and NSTC 114-2218-E-006-021.}}
\thanks{\textcolor{black}{National Cheng Kung University's Institutional Review Board (IRB) reviewed and approved all the study procedures.}}
\thanks{\textcolor{black}{The authors are with the Department of Mechanical Engineering, National Cheng Kung University, Tainan, Taiwan. (e-mail: \texttt{yliu@mail.ncku.edu.tw})}}\par
}

%
%

\markboth{Arxiv.org version}%
{Shell \MakeLowercase{\textit{et al.}}: Bare Demo of IEEEtran.cls for IEEE Journals}
%



\maketitle
\begin{abstract}
Mobile robots in public spaces must ensure pedestrians' comfort, and yet empirical studies of walkers' subjective safety are rare.
\textcolor{black}{Many classical navigation algorithms do not distinguish the walkers from dynamic obstacles and do not explicitly model subjective human factors.}
Moreover, most studies focus on holonomic mobile robots, whereas applications demand Nonholonomic Mobile Robots (NMR). 
This paper develops socially aware algorithms for NMRs, proves the stability, verifies the performance experimentally, and statistically analyzes the reported comfort.
We design a framework for NMRs using Social Force Model (SFM) and the projected Time-to-collision Social Force Model (TSFM).
We formalize the NMR-pedestrians' and NMR-obstacles' interactions and prove the system's stability, assuming boundedly nonpassive pedestrians.
Simulations calibrate the models by maximizing a hybrid cost function of comfort and speed. 
Pedestrian-robot interaction experiments compare SFM and TSFM to two remote-controlled baselines and collect walkers' reported comfort.
Statistical tools analyze survey results collected during the experiments.
Benchmarking the algorithms against previous studies highlights the proposed methods' advantage with respect to the studied metrics.
Overall, the models are stable and improve pedestrian comfort when an NMR navigates through a pedestrian crowd.
\end{abstract}

\begin{IEEEkeywords}
Pedestrian-Robot Interaction Stability, Pedestrian Comfort, Socially-Aware Navigation, Nonholonomic Mobile Robots (NMR), Social Force Model (SFM), Projected Time-to-Collision (PTTC).
\end{IEEEkeywords}
%
\IEEEpeerreviewmaketitle
\section{Introduction}\label{sec:intro}
\IEEEPARstart{S}{ocially-aware} mobile robots on shared public spaces must care for the pedestrians' comfort or subjective safety in addition to their objective safety~\cite{He2025}.
Objective safety focuses on avoiding physical collisions with pedestrians, whereas subjective safety focuses on the feeling of being safe.
\textcolor{black}{Many classical navigation algorithms treat the walkers and dynamic obstacles the same}, for example~\cite{Lee2024, Xie2023}.
Perhaps the reason lies in the difficulty of quantifying human discomfort by the robot, preventing the development of implementable algorithms~\cite{Francis2025}.
\par
Mobile robots inevitably join futuristic public spaces.
These shared spaces are usually unstructured, and the walkers do not follow strict guidelines. 
Without considering walkers' movement patterns in the navigation algorithm, cruising through a crowd is challenging~\cite{Brscic2017}.
Therefore, the interaction with human crowds requires further development in traditional obstacle avoidance~\cite{Satake2013}, for example, human-robot cooperative path-planning~\cite{Trautman2015}.
Social Force Models (SFM)~\cite{JAFARI2024-SIMPAT}, Velocity-Obstacle (VO)~\cite{Fiorini1998} and its variant Optimal Reciprocal Collision Avoidance (ORCA)~\cite{vandenBerg2011}, and learning-based methods~\cite{Le2024} are the main candidates for mobile robot navigation in pedestrian crowds~\cite{VanDerMeer2025}.
Apart from the mainstream research, the game theory and mixed strategy Nash equilibrium structurally match mobile robot-crowd interaction dynamics~\cite{Sun2026}; however, the framework for socially aware navigation remains largely understudied.
\par
SFM, initially designed for crowd panic evacuations~\cite{Helbing1995}, is a promising candidate for mobile robot interaction with pedestrians.
In mobile robot-pedestrian interactions, SFM is usually used to model and predict the pedestrian trajectory, for example~\cite{kidokoro2015}.
\textcolor{black}{We focus on using SFM as the mobile robot navigation algorithm, rather than predicting other agents' movements.
Since SFM is a pedestrian walking model, the core motivation for using it as the navigation algorithm is to imitate human movement patterns.
The assumption is that if a robot moves like a pedestrian, the surrounding pedestrians feel more comfortable around it.}
\par
Recently, Liu et al. proposed an SFM variant for e-scooters~\cite{Liu2022-T-ITS}. 
Their further research extends the model to other personal mobility vehicles, including mobile robots~\cite{JAFARI2024-SIMPAT, JAFARI2024-3-SORO}.
For instance, they propose a predictive SFM variant with occasionally visible improvements in pedestrian comfort~\cite{Jafari2026-2-aim}.
Similarly, Shiomi et al. improve pedestrians' experience around the mobile using an SFM~\cite{Shiomi2014}.
The robot reproduces human-like movement patterns.
Ægidius et al. propose an augmented SFM for guiding mobile robots through human crowds~\cite{AEgidius2024}.
Using a quadruped robot, they experimentally improve social navigation metrics.
Another example is the work by Kamezaki et al. that uses inducible SFM and reactively cruises in a crowd~\cite{Kamezaki2022}.
Ferrer et al. calibrate and use a basic holonomic and circular SFM on a nonholonomic setup~\cite{Ferrer2013}.
In addition, Brayan et al. collect a dataset for social navigation using a holonomic robot and an NMR~\cite{Brayan2026}.
Then, they calibrate a distance-based circular SFM and compare the trajectories from the dataset.
\par
\textcolor{black}{Humans, having no velocity constraints, are naturally holonomic agents, while most practical mobile robots are nonholonomic.
Thus, most SFM research uses holonomic mobile robots, whereas applying the human-like holonomic model to physically constrained mobile robots remains an ongoing challenge. 
This paper develops a framework for nonholonomic robots to apply SFM and its variants.
Moreover, our TSFM variant incorporates time-to-collision into its formulation to relate more directly to pedestrian comfort and empirically compares SFM, TSFM, and two remote-controlled scenarios.}
\par
ORCA is another mobile robot navigation algorithm that evolved from Velocity-Obstacle (VO) and Reciprocal Velocity Obstacle (RVO).
It shares the avoidance responsibility between the agents and realizes safe velocities at each instance~\cite{vandenBerg2011}.
Martinez-Baselga et al. present AdaptiVe Optimal Collision Avoidance Driven by Opinion (AVOCADO) and relax the reciprocity assumption of VO~\cite{Martinez2025}.
AVOCADO uses opinion dynamics and adaptively estimates the cooperation share of the neighboring agents.
Levy et al.~\cite{Levy2015} use a robust collision avoidance system for a robot swarm.
They utilize an ORCA extension to ensure safe and natural robot movements in a crowded environment.
Samavi et al. introduce SICNav, an MPC path-planner that uses ORCA to predict pedestrian movements~\cite{Samavi2025}.
They experimentally evaluate ORCA's performance in pedestrian walking pattern replication.
\par
\textcolor{black}{Most VO variants, such as ORCA, treat pedestrians as dynamic obstacles and therefore are designed to guarantee objective safety.
Since objective safety correlates with subjective safety, such guarantees may also improve pedestrian comfort.
However, the algorithms are not specifically designed to care for pedestrian feelings.}
\par
Learning-based navigation methods are turning into the primary focus in mobile robot navigation research.
Hirose et al. train robots to move through a crowd while minimizing their presence impact on the crowd~\cite{Hirose2024}.
Liu et al. introduce a deep reinforcement learning approach for robot navigation in crowded environments~\cite{Liu2020-2}.
They process static and dynamic obstacles, improving motion planning and reporting high performance in simulated and real-world scenarios.
Flogel et al. propose a Deep Reinforcement Learning (DRL) framework where the robot's actions emerge from interactions rather than predefined rules~\cite{Flogel2024}.
\par
Learning-based studies also improve generalization and social awareness.
Sen et al. enhance the generalization capabilities of human-aware robot navigation using Domain Randomization (DR) techniques by widening the pedestrian behavior range in simulations~\cite{Sen2025}.
They train navigation policies using the broader pedestrian behavior and stand out among robots trained using ORCA pedestrians within a narrow behavior spectrum.
Chen et al. present a method for socially aware Object Goal Navigation (ObjectNav) in dynamic environments using a heterogeneous scene representation learning approach~\cite{Chen2024-3}.
The approach models human-robot-object ternary interactions to maintain feature specificity.
It uses a DRL strategy to navigate efficiently by predicting state transitions and estimating future states.
\par
\textcolor{black}{Standard learning-based methods lack formal stability analysis by construction, in contrast to classical control approaches, where stability must be proven using tools such as Lyapunov theory.
These methods rely on navigation policy, a high-dimensional function like neural networks, learned from data or interaction.
The function doesn't have a structured formulation, and control stability analysis relies on explicit mathematical formulations and constraints.
Therefore, by construction, an analytical stability proof for the standard learning methods is unseen.
Experiments validate safety and convergence for learning methods, and there are always risks, especially when facing unseen scenarios or distribution shifts in reality.
That being said, formal guarantees are discussable in learning methods by augmenting learning-based algorithms with other mechanisms such as control barrier functions~\cite{Cheng2019}, Lyapunov-based critics~\cite{Chow2018}, reachability analysis~\cite{Wang2024}, or policy verification/safety shielding~\cite{Alshiekh2018}.
However, these are beyond standard DRL and require hybrid designs.}
\par
Human discomfort and subjective safety are \textcolor{black}{relatively underexplored areas in mobile robot-human interactions compared to objective safety and collision avoidance.}
Notably, a few research studies have examined and addressed human discomfort.
For example, Neggers et al. experimentally study the geometry and the dimensions of human personal space~\cite{Neggers2022}.
They provide discomfort contours around people interacting with holonomic mobile robots.
Hoang et al.'s study estimates a socially optimal path to approach a group of people~\cite{Hoang2023}.
On its way to the group, the robot avoids pedestrians and their dynamic social zones.
Greenberg et al. highlight pedestrian optimism toward mobile robots when their trajectories have mild curvatures~\cite{Greenberg2025}.
Jang and Ghaffari simulate a mobile robot navigating around a human~\cite{Jang2024}.
The robot uses barrier functions to avoid entering experimentally identified social zones.
Jafari et al. study the effect of kinematic variables on subjective safety and suggest a handcrafted estimator for pedestrian comfort~\cite{Jafari2026-1-icra}.
\textcolor{black}{We measure subjective and descriptive aspects of the pedestrian feelings using four questions (Q1--Q4) that capture the perceived comfort, smoothness, distance, and speed, respectively.
In this paper, we use the terms ``comfort" and ``subjective safety" interchangeably for the pedestrians' overall evaluation of the interaction Q1.}
\par
Time-to-Collision (TTC) is a well-known, long-developed metric for near-miss severity in traffic safety~\cite{Hayward1972}.
In robotics research, however, the TTC application is quite recent.
Masaki et al. use TTC as an indicator of collision possibility. 
In their study, the force generated by the TTC assists the remote operator in managing a mobile robot~\cite{Masaki2020}.
Shahriari and Biglarbegian quantify the collision immediacy among mobile robot crowds using TTC and suggest a Lyapunov-stable controller for heterogeneous mobile robot groups~\cite{Shahriari2022}.
Jafari and Liu model a heterogeneous crowd movement pattern on a futuristic sidewalk using a TTC extension~\cite{JAFARI2024-3-SORO}.
In an e-scooter-pedestrian safety study, Projected Time-to-Collision (PTTC), a variant of TTC, correlates with pedestrian subjective safety~\cite{JAFARI2024-2-NATCOM}.
If the correlation extends to mobile robots, online estimations of nearby pedestrians' comfort are possible.
\par
\textcolor{black}{Previous research mainly focuses on trajectory-based path planning/tracking or local re-planning using predefined trajectories, for example~\cite{Zou2022, Fox1997, Ma2025-2, Tarantos2024}, which is not the case for sidewalks where the pedestrians are dynamic and unpredictable over long horizons.}
In addition, when the research considers pedestrians and a dynamic environment, the mobile robot is often holonomic. 
However, major practical applications are interested in Nonholonomic Mobile Robots (NMRs) on shared spaces because of their higher speed and lower maintenance~\cite{Taheri2020}. 
Examples are delivery services and sweeping machines.
Although a few previous studies use NMRs for public space experiments, the underlying theory assumes holonomic mobile robots~\cite{Repiso2024, Zanlungo2017, Shiomi2014, Mac2025, Vasquez2014}.
A typical inconsistency is the assumption of a centered mass and double-integrator dynamics for NMRs.
A conventional SFM often requires impossible lateral movements on a nonholonomic platform. 
For instance, Ferrer et al. directly apply the conventional centered-mass double-integrator SFM to an NMR~\cite{Ferrer2013}, ignoring the inconsistency.
\par
We provide a framework for Artificial Potential Field methods (APF), specifically SFM, to be implemented on NMRs.
Moreover, we propose a PTTC-based SFM (TSFM) to improve pedestrian comfort around the NMR.
\textcolor{black}{The SFM and TSFM comparison isolates the effect of the PTTC integration into TSFM.}
A hybrid cost function calibrates both models in simulations.
The experiment results show that TSFM surpasses SFM in the walkers' comfort metric but falls behind in the robots' speed criterion.
\par
Additionally, the mobile robot-pedestrian interaction studies often lack stability analysis. 
\textcolor{black}{
A mobile robot entering a pedestrian crowd disturbs the harmonious multi-agent environment. 
Due to the collision-avoidance algorithms, the mobile robot itself reacts to the disturbed flow.
Therefore, unbounded robot velocities and divergent system states are possible, posing physical safety risks to nearby walkers.
Thus, the stability analysis of robot-pedestrian interaction is critical.
In this paper, the robot-pedestrian interaction is stable if the interaction dynamics states are all bounded over time.
}
We formalize the interaction and lay the foundations for stability analysis under the key assumption of bounded non-passivity for pedestrians.
\par
In addition, we compare SFM and TSFM with a previous learning-based study (DRL) and remote-controlled trials by an extensive set of experiments.
The remote-controlled cases and another previous research (ORCA) on pedestrian comfort form the comparison baselines.
Moreover, since mobile robot-pedestrian studies rarely consider pedestrian comfort in experiments, we use questionnaires to benchmark the algorithms' comfort level empirically.
The statistical analysis suggests that TSFM is more comfortable than SFM. 
Both autonomous algorithms are inferior to remote-controlled interactions.
Pedestrian ratings analysis hints that the movement smoothness is the underlying reason.
\par
Briefly, this paper presents a framework for force-based navigation algorithms specialized for NMRs and proves the interaction stability.
Moreover, we experimentally compare the algorithms with previous research using comfort metrics.
Survey responses compare the methods and verify the pedestrian comfort improvement.
The paper's contributions are as follows:
\begin{enumerate}
    \item The paper develops a framework for force-based navigation algorithms for NMRs.
    The framework applies beyond the SFM and to any other APFs.
    Previous research either fully focuses on holonomic mobile robots or develops a holonomic algorithm and applies it to an NMR during experiments.
    We apply SFM to an NMR as proof of concept.
    Previous research implements SFM on robots in a pedestrian-like manner, ignoring the NMR's kinematic constraints.
    The framework allows for the consideration of kinematic constraints in an SFM-based NMR navigation.
    The successful implementation of SFM on an NMR proves the concept.
    \item The paper introduces a PTTC-based SFM (TSFM) and improves pedestrian experience around mobile robots.
    TSFM modeled e-scooter rider movement patterns in previous research. 
    This paper extends the novelty to mobile robots, resulting in a more pleasant walk for pedestrians in shared sidewalks.
    \item The paper lays the foundations for stability analysis of mobile robot-pedestrian interactions.
    Assuming pedestrians' bounded non-passivity, the stability analysis shows uniform boundedness of the system.
    Previous research focused on obstacle avoidance and has not studied the stability in mobile robot-pedestrian interactions and \textcolor{black}{mainstream} learning-based methods do not have formal stability guarantees.
\end{enumerate}
In addition to the theoretical contributions, the paper presents the following practical novelties:
\begin{enumerate}
    \item The comparison with a previous learning-based algorithm highlights the advantages of SFM and TSFM.
     The inclusion of remote-controlled cases provides insights into the models' performance.
     \item We compare the algorithms' performance in pedestrian comfort using surveys. 
    Statistical tools benchmark the algorithms against each other and two remote-controlled cases.
\end{enumerate}
In addition, outdoor multi-pedestrian tests verify the algorithms on actual sidewalks.
In summary, the major contributions of the paper are the study of stability in the presence of pedestrians and the statistical analysis of pedestrian comfort.
\par
The paper structure is as follows.
Section~\ref{sec:back} reviews the necessary background.
Section~\ref{sec:main} introduces the main theoretical contributions.
Section~\ref{sec:stability} proves the system stability.
Section~\ref{sec:cal} calibrates the models.
Section~\ref{sec:exp} details the experiments.
Section~\ref{sec:res} presents the results and discusses the outputs.
Section~\ref{sec:con} outlines the paper and suggests follow-ups.
\section{Background}\label{sec:back}
This section reviews the required background for further building upon in the paper.
It briefly introduces the SFM's basics, introduces TTC, and details the PTTC as a variant of TTC in 2-dimensional space~\cite{JAFARI2024-3-SORO,JAFARI2024-2-NATCOM}.
\subsection{Social force model}\label{sec:back:sfm}
Helbing et al. formulated the pedestrian micro movements in a crowd using SFM~\cite{Helbing1995}.
Over the last thirty years, SFM has become the dominant tool in crowd dynamics simulations~\cite{Ma2025}.
Moreover, a few studies implemented it on Personal Mobility Vehicles (PMVs)~\cite{Dias2017, Dias2018, HASEGAWA2018} and mobile robots~\cite{Shiomi2014, Truong2017, Repiso2024, Kamezaki2022, Mac2025}. 
All these studies used the pedestrian model directly on the mobile robot, ignoring the kinematic constraints.
\par
SFM assumes that the movement of people in a shared space is the result of their interactions and reactions to walls, curbs, obstacles, PMVs, etc.
It formulates the interactions by a set of psychological forces, creating a virtual force space.
The superposition of these forces adds linearly and forms an equivalent total psychological force $\vec{F}_i$ exerted on pedestrian $i$,
\begin{equation}
\vec{F}_{i}=\sum_{j=1}^{N}\vec{f}_{soc,ij}+\sum_{k=1}^{B}\vec{f}_{bnd,ik}+\vec{f}_{des,i}\;,
\label{eq:SFMForce}
\end{equation}
where $\vec{f}_{soc,ij}$ is a social repulsive force designed to replicate pedestrian $i$'s intention to keep a distance from pedestrian $j$ and $N$ is the number of nearby pedestrians.
Similarly, $\vec{f}_{bnd,ik}$ is a repulsive force to prevent the pedestrian $i$ from getting too close to the boundary or wall $k$; $B$ is the number of boundaries in $i$'s vicinity.
$\vec{f}_{des,i}$ is the an attractive force that pulls the pedestrian $i$ toward its destination with its desired velocity.
\par
Individual $i$'s movement pattern is the result of its reaction to $\vec{F}_i$.
Conventional SFM assumes that pedestrians and the mobile robots are centered masses. 
Therefore, their movement patterns follow a double integrator format
\begin{equation}
\vec{F}_{i}=m_i\vec{a}_i,
\label{eq:F=ma}
\end{equation}
where $m_i$ and $\vec{a}_i$ are the agent's virtual mass and acceleration, respectively.
The framework suggested in Section~\ref{sec:SFM:framework} uses the robot's dynamics instead.
Sections~\ref{sec:SFM:SFM} and~\ref{sec:SFM:TSFM} formulate the forces in detail for SFM and TSFM, respectively.
\subsection{Projected time-to-collision}
\begin{figure}[t]
\centering
\includegraphics[width=0.48\textwidth]{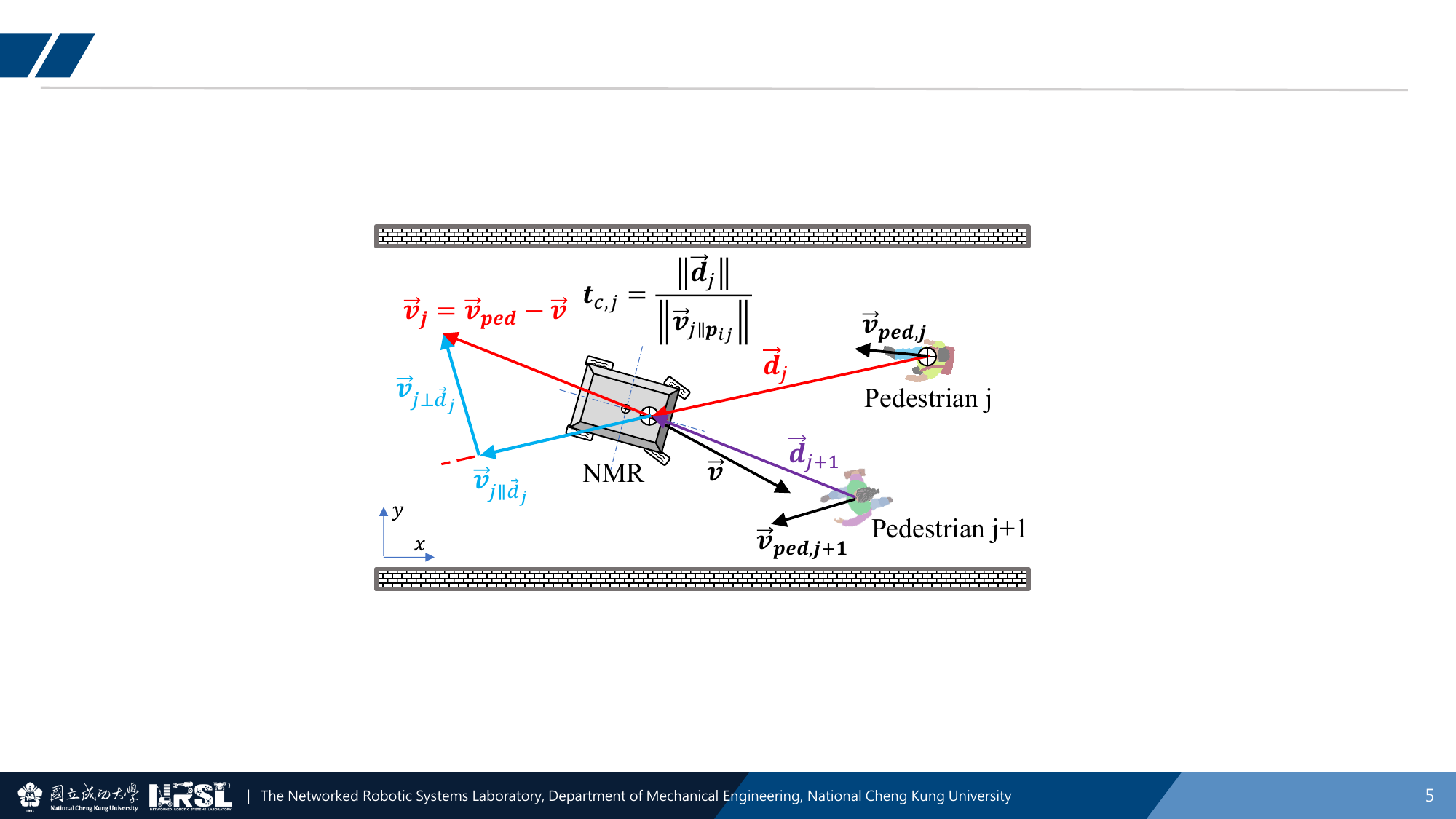}
\caption{PTTC concept when a robot passes through a crowd. The robot detects pedestrians $j$ and $j+1$. It measures $\vec{d}_j$ and $\vec{v}_j$ and computes $t_{c,j}$ and $\vec{f}_{soc,j}$ for pedestrian $j$. Then, it iterates for pedestrian $j+1$.}
\label{fig:PTTC}
\end{figure}
TTC is the remaining time until two objects collide.
On a straight line and assuming the objects maintain their current velocity, it is $\frac{\Delta X}{\Delta V}$, where $\Delta X$ and $\Delta V$ are their relative distance and relative velocity, respectively.
However, in a 2-dimensional space, the objects may not collide if their trajectories are skewed.
Still, the rate at which the two agents approach each other on their line of sight affects their avoidance behavior~\cite{JAFARI2024-3-SORO}.
\par
Jafari and Liu defined PTTC as a pedestrian discomfort metric for e-scooters~\cite{Jafari2023-2-IFAC}.
They correlated the human-reported discomfort with PTTC in pedestrian interaction with e-scooters~\cite{JAFARI2024-2-NATCOM}.
Moreover, they proposed a model describing e-scooter rider movement patterns on a sidewalk using a PTTC variant~\cite{JAFARI2024-3-SORO}.
We extend the PTTC application to mobile robots, specifically NMRs, and propose the PTTC-based Social Force Model (TSFM) as a navigation algorithm for NMRs on sidewalks.
For the sake of completeness, we briefly introduce PTTC; please see~\cite{JAFARI2024-2-NATCOM, JAFARI2024-3-SORO} for a detailed description.
\par
Fig.~\ref{fig:PTTC} describes the PTTC concept when the robot passes through a pedestrian crowd.
The rate at which the robot and the pedestrian $j$ approach each other is the projection of their relative velocity $\vec{v}_j$ onto their relative distance $\vec{d}_j$,
\begin{align}
\vec{v}_{j\parallel\vec{d}_j}&=\frac{\vec{d}_j\cdot\vec{v}_j}{\|\vec{d}_j\|},\label{eq:v_par_d}\\
\vec{v}_j&=\vec{v}_{ped,j}-\vec{v},\label{eq:vj}
\end{align}
where $\vec{v}_{j\parallel\vec{d}_j}$ is the projection and $\vec{v}_{ped,j}$ is the pedestrian $j$'s velocity vector in the global frame.
$\vec{v}_j$ and $\vec{d}_j$, the relative velocity and position, may be interpreted as the pedestrian's absolute velocity and position in the robot's local frame as well. 
The onboard depth camera measures $\vec{v}_j$ and $\vec{d}_j$ directly, eliminating the need for a global frame.
\par
$t_{c,j}$ denotes PTTC to pedestrian $j$ on the sidewalk.
$t_{c,j}$ is the time that the pedestrian perceives as the time remaining to the collision if the distance on the line of sight  $\vec{d}_j$ keeps decreasing at the current rate $\vec{v}_{j\parallel\vec{d}_j}$. Thus,
\begin{equation}\label{eq:tcj} 
t_{c,j}=\frac{\|\vec{d}_j\|}{\vec{v}_{j\parallel\vec{d}_j}}=\frac{\|\vec{d}_j\|^2}{\vec{d}_j\cdot\vec{v}_j}.   
\end{equation}
The robot assigns a PTTC to all detected pedestrians on the sidewalk.
\section{Social Force Models for Nonholonomic Mobile Robots}\label{sec:main}
This section develops a framework for NMRs that use force-based models.
The framework produces movements that are compatible with the robot's kinematic constraints in response to the exerted virtual forces.
We introduce two algorithms: the mainstream formulation for SFM and a novel formulation, TSFM.
\subsection{Nonholonomic framework}\label{sec:SFM:framework}
\begin{figure}[t]
\centering
\includegraphics[width=0.48\textwidth]{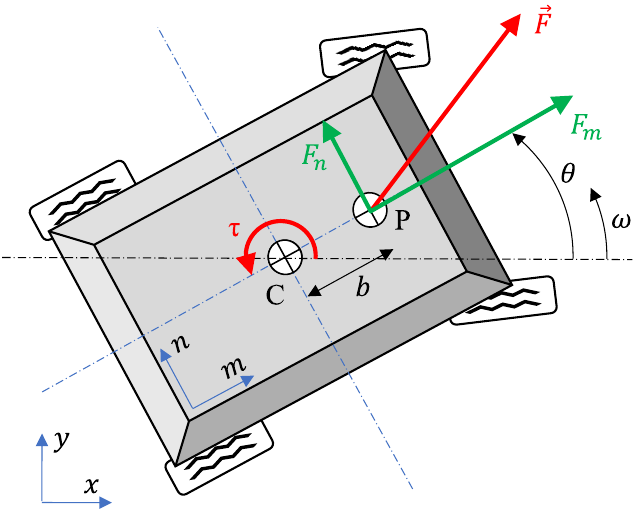}
\caption{The mobile robot geometry, the effective force $\vec{F}$, and its projections $F_m$ and $F_n$ in the local coordinates.
The point of force application P is a design choice. Camera measurements must be transformed to this point.}
\label{fig:DDNMR}
\end{figure}
A novelty of this paper is a framework for applying SFM and its variants to NMRs.
The framework applies to holonomic mobile robots with proper modifications.
However, the paper focuses on a more challenging case of NMRs.
Here, we construct the mathematical foundation using a differential-drive mobile robot. 
Nevertheless, the approach is extendable to other NMRs by applying the corresponding kinematic constraints. 

\par
To generate pedestrian movement patterns, SFM assumes that psychological forces are physical forces exerted on people as if they were centered masses~\cite{Helbing1995}.
Our framework adopts the idea of assuming virtual forces as physical forces from SFM. 
However, instead of applying it to centered masses, it applies it to NMRs and considers their kinematic constraints.
\par
The navigation algorithms require the NMR to react to a virtual force.
The algorithms calculate the virtual force, apply it to the robot's dynamics under the constraints, and get the robot's virtual reaction.
Then, it sends the virtual reaction as a command to the robot, and the robot behaves as if it is subjected to an equivalent physical force.
Therefore, the framework generates a feasible reaction to virtual forces for NMRs under kinematic constraints.
Consequently, the robot behaves as if it's being dragged by $\vec{F}$. 
\par
\noindent \textbf{Kinematics.}~
Fig.~\ref{fig:DDNMR} shows a differential-drive NMR and the virtual force vector, $\Vec{F}$. 
Its center of rotation is $C$.
The distance between the virtual force exertion point $P$ and $C$ is the torque arm $b$, assuming lateral symmetry.
$x$ and $y$ are $C$'s global coordinates and $\theta$ shows the robot orientation.
$v$ and $\omega$ stand for its linear and angular speeds; to avoid confusion, we use the term ``speed" referring to the scalar and ``velocity" to the vector.\par
We define $q=[x, y, \theta]^T$ and $z=[v, \omega]^T$;
\begin{equation}
\dot{q}=Jz, \text{with~}J=
\begin{bmatrix}  
  \cos \theta & 0 \\
  \sin \theta & 0 \\
  0           & 1
\end{bmatrix}.
\label{eq:J=[]}
\end{equation}
The no-lateral-movement kinematic constraint is 
\begin{equation}
A\dot{q}=0
\label{eq:Aq=0}
\end{equation}
with $A=[-\sin{\theta}, \cos{\theta}, 0]$ and $\dot{q}=[\dot{x}, \dot{y}, \dot{\theta}]^T$.
Note that
\begin{equation}
J^TA^T=0_{2 \times 1}.
\label{eq:JA=0}
\end{equation}
\par
\noindent \textbf{Dynamics.}~
The robot's mass and rotational inertia around $C$ are $m$ and $I$.
In Fig.~\ref{fig:DDNMR}, $\Vec{F}$ is the virtual force dictating the movement patterns and $F_m$ and $F_n$ are its components in the robot's local frame $m-n$.
$\tau$ is a virtual torque for further robot behavior tuning.
In this paper, we set $\tau=0$.
The dynamic model of an NMR is~\cite{Fierro1995}
\begin{equation}
M\Ddot{q}+C\dot{q}=Bu-A^T\lambda_L,
\label{eq:Dynamics}
\end{equation}
where $M$, $C$, $B$, and $A$ are inertia, Coriolis and centrifugal, input transformation, and Pfaffian constraint matrices.
$\lambda_L$ is a Lagrange multiplier enforcing the constraints and 
\begin{equation}
u=\vec{F}=
\begin{bmatrix}  
  F_m \\
  F_n 
\end{bmatrix}.
\label{eq:u=[]}
\end{equation}
Multiplying \eqref{eq:Dynamics} in $J^T$ from left and using \eqref{eq:Aq=0} results in
\begin{equation}
J^TM\Ddot{q}+J^TC\dot{q}=J^TBu.
\label{eq:DynamicsReduced}
\end{equation}
Using \eqref{eq:J=[]} and its derivative $\Ddot{q}=J\dot{z}+\dot{J}z$, the dynamics reduces to
\begin{equation}
M_r\dot{z}+C_rz=B_ru,
\label{eq:DynamicsFinal}
\end{equation}
where $M_r=J^TMJ$, $C_r=J^TM\dot{J}+J^TCJ$, and $B_r=J^TB$.
For details on the matrices, see Appendix~\ref{sec:app:matrices}.
\par
\begin{remark}\label{remark:ParSel}
Since we are setting up a virtual force framework, the use of actual $m$, $I$, and $b$ is not required.
For example, selecting small $m$ and $I$ results in translational and rotational agility. 
However, we use the actual physical values so that the resulting virtual force matches the equivalent physical system, easing the interpretation of the results.
\end{remark}
The proposed framework applies the virtual forces using $u$ in \eqref{eq:DynamicsFinal}, obtains $z$ at each time step, and sends it as a command to the robot. 
Therefore, the robot reacts to virtual forces considering the kinematic constraints.
\subsection{SFM formulation}\label{sec:SFM:SFM}
\begin{figure}[t]
\centering
\includegraphics[width=0.48\textwidth]{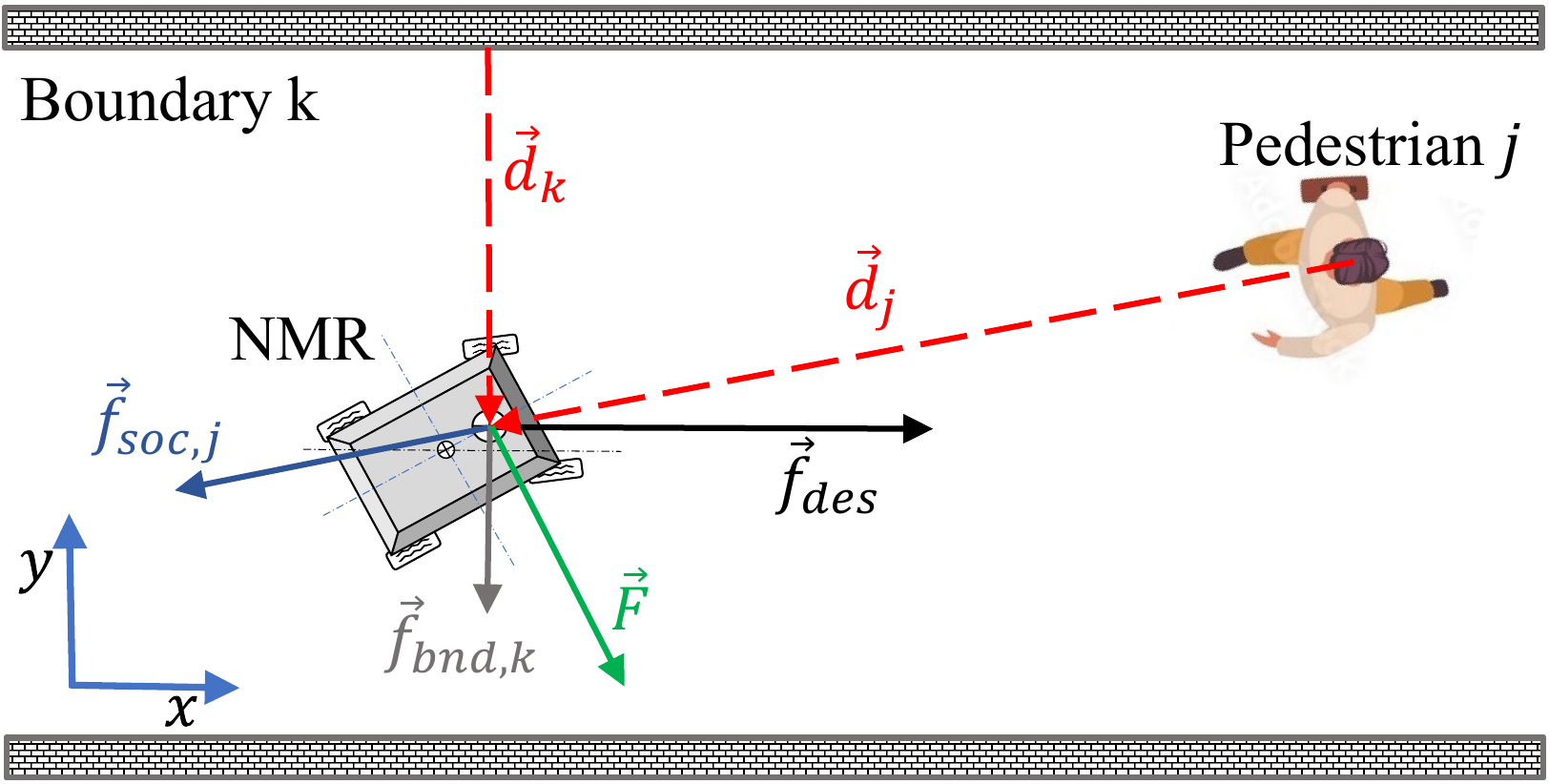}
\caption{An NMR moves in response to $\vec{F}$. 
$\vec{F}$ is the linear superposition of boundary forces, social forces, and desired force; see~\eqref{eq:SFMForce}.}
\label{fig:SFMNMR}
\end{figure}
In the proposed framework, $\Vec{F}$ drives the NMR and determines its motion.
Following the pedestrian SFM formulation \eqref{eq:SFMForce}, this section defines the social force $\Vec{f}_{soc,j}$, the boundary force $\vec{f}_{bnd,k}$, and the desired force $\Vec{f}_{des}$.
Since we study a single robot on a sidewalk, we drop the subscript $i$.
See Fig.~\ref{fig:SFMNMR} for a visual representation of the following.\par
$\Vec{f}_{soc,j}$ is the force applied to the robot to avoid pedestrians. 
The most common formulation for the social force is \cite{Helbing1995}
\begin{equation}
\Vec{f}_{soc,j}\left(\Vec{d}_j,\phi_j\right)=\Vec{g}\left(\Vec{d}_j\right)\gamma\left(\phi_j\right),
\label{eq:f_soc}
\end{equation}
where $\vec{d}_j$ is the relative position of the robot with respect to the pedestrian $j$ and $\phi_j$ is their encounter angle~\cite{Liu2022-T-ITS}.
The function $\vec{g}$ is
\begin{align}
\vec{g}(\vec{d}_j)&=\alpha \exp({\frac{-\|\vec{d}_j\|}{\beta}})\vec{e}_j.\label{eq:g}
\end{align}
The constants $\alpha$ and $\beta$ determine the $\vec{g}$'s magnitude and range, respectively.
$\vec{e}_j$ is $\vec{d}_j$'s unit vector.\par
The function $\gamma$ anisotropizes $\vec{g}$, making it more sensitive to pedestrians in front of the robot,
\begin{align}
\gamma\left(\phi_j\right)&=\lambda+(1-\lambda)\frac{1+\cos{\phi_j}}{2},\label{eq:gamma}\\
\cos{\left(\phi_j\right)}&=\frac{\vec{v}_j}{\|\vec{v}_j\|}\cdot\frac{-\vec{d}_j}{\|\vec{d}_j\|},
\label{eq:cosphi}
\end{align}
where $\lambda$ determines the anisotropy degree; $\lambda=1$ is isotropic and $\lambda=0$ is fully anisotropic.
$\vec{v}_j$ is the relative velocity of the agent $j$ with respect to the NMR.
\par
The robot reacts to forces from boundaries, such as curbs or walls, to avoid collisions with them.
The boundary force is
\begin{equation}
\vec{f}_{bnd,k}\left(\Vec{d}_k\right)=\alpha_{b}\exp{\left(-\frac{\|\vec{d}_k\|}{\beta_{b}}\right)}\vec{e}_k,
\label{eq:f_bnd}
\end{equation}
where $\vec{d}_k$, shown in Fig.~\ref{fig:SFMNMR}, connects the boundary to the point of force application and is normal to the boundary; $\vec{e}_k$ is its unit vector.
The constants $\alpha_b$ and $\beta_b$ determine the $\Vec{f}_{bnd,k}$'s magnitude and range, respectively.\par
The desired force is the force that drives the robot to its destination, letting it deliver the required service.
Following~\cite{Helbing1995},
\begin{align}
\vec{f}_{des}=m\frac{\vec{v}_{des}-\vec{v}}{\tau_d},
\label{eq:f_des}
\end{align}
where $\vec{v}_{des}$ and $\vec{v}$ are the robot's desired velocity and instantaneous velocity, respectively.
Parameter $\tau_d$ adjusts the robot's aggressiveness in reaching its desired velocity.\par
Most research assumes that $\vec{v}_{des}$ is constant both for pedestrians~\cite{Helbing1995} and for robots~\cite{Repiso2024}.
While the desired velocity is quasi-static for SFM pedestrians, it varies considerably for other agents~\cite{Liu2022-T-ITS}.
In our experience, the assumption leads to excessively aggressive robots in crowded areas.
Therefore, we consider varying $\vec{v}_{des}$~\cite{JAFARI2024-SIMPAT},
\begin{align}
\vec{v}_{des}&=\left(\exp{\left(-\frac{S}{\sigma}\right)}\left(v_{max}-v_{min}\right)+v_{min}\right)\vec{e}_{des},\label{eq:v_des}\\
S&=\sum_{j=1}^{N}\|\vec{f}_{soc,j}\|+\sum_{k=1}^{B}\|\vec{f}_{bnd,k}\|,
\label{eq:S}
\end{align}
where $v_{max}$ and $v_{min}$ are the maximum and minimum speeds occurring in free motion ($S=0$) and passing through a dense crowd ($S=\infty$), respectively.
Constant $\sigma$ adjusts the desired speed; $\vec{e}_{des}$ is the desired direction's unit.
This paper sets $v_{max}=2.5~m/s$ and $v_{min}=0~m/s$.
Thus, the SFM parameters that require calibration are 
\begin{align}
\Gamma_{SFM}=\left[\alpha, \beta, \Delta t, \lambda, \alpha_b, \beta_b, \tau_d, \sigma\right].
\label{eq:SFMParset}
\end{align}
\subsection{TSFM formulation}\label{sec:SFM:TSFM}
In this section, the goal is to develop an algorithm that takes pedestrians' comfort into account, rather than just avoiding them as a dynamic obstacle.
Since PTTC correlated with pedestrian comfort when interacting with e-scooters on sidewalks~\cite{JAFARI2024-2-NATCOM}, we propose to use PTTC, instead of $\vec{d}_j$, to determine the robot's avoidance behavior. 
In TSFM-based algorithm, we use~\eqref{eq:f_bnd} and~\eqref{eq:f_des} for $\vec{f}_{bnd,k}$ and $\vec{f}_{des}$ and update $\vec{f}_{soc,j}$.
In SFM, $\vec{f}_{soc,j}$ mainly depends on $\vec{d}_j$.
The relative velocity $\vec{v}_j$ appears only in the anisotropy adjustment~\eqref{eq:cosphi}.
\par
In TSFM, 
\begin{align}
\vec{f}_{soc,j}=\exp{\left(-\frac{T_j-T_{cr}}{\beta_T}\right)}\vec{e}_{j},
\label{eq:f_soc:TSFM}
\end{align}
where $T_j$ is the PTTC between the robot and the pedestrian $j$ \textcolor{black}{obtained using~\eqref{eq:tcj}}, $\beta_T$ is the range tuner, and $\vec{e}_{j}$ is $\vec{d}_j$'s unit vector.
$T_{cr}$ adjusts the force magnitude, like $\alpha$ and $\alpha_b$ in \eqref{eq:g} and \eqref{eq:f_bnd}.
Note that \eqref{eq:f_soc:TSFM} is equivalent to
\begin{equation}
\Vec{f}_{soc,j}=\alpha_{T}\exp{\left(-\frac{T_j}{\beta_{T}}\right)}\vec{u}_{j},
\label{eq:f_soc:TSFMeq}
\end{equation}
with $\alpha_T=\exp{\left(\frac{T_{cr}}{\beta_T}\right)}$.
TSFM parameters that require calibration are 
\begin{align}
\Gamma_{TSFM}=\left[T_{cr}, \beta_T, \alpha_b, \beta_b, \tau_d, \sigma\right].
\label{eq:TSFMParset}
\end{align}
\section{Stability Analysis}\label{sec:stability}
This section analyzes the stability of the system introduced in Section~\ref{sec:main}.
First, Theorem~\ref{theorem:asymptoticallyStable} assumes a constant $\vec{v}_{des}$ and proves the asymptotic stability using Lemmas~\ref{lemma:UUB1} and~\ref{lemma:UUB2}.
Then, Theorem~\ref{theorem:UUBStable} relaxes the assumption and proves that the system is Uniformly Ultimately Bounded (UUB).
Since the stability analysis relies on the passivity concept in the context of two-port networks, Theorems~\ref{theorem:Nonpassivity_f_des} and~\ref{theorem:wall_nonpassivity} prove that the robot attraction to a goal and the static obstacles/boundaries repulsion terms are boundedly non-passive.
\par
The key challenge in interaction stability analysis is the arbitrary pedestrian movements and difficulties in modeling their intentions.
A pedestrian may destabilize the system by injecting infinite energy into it.
Assumption~\ref{assumption:bounded_nonpassivity} circumvents the challenge by restricting the pedestrian's injected energy to the system and formalizes bounded non-passivity for the walkers.
Finally, Theorem~\ref{theorem:boundedness} proves the boundedness of the whole system relying on Theorems~\ref{theorem:Nonpassivity_f_des} and~\ref{theorem:wall_nonpassivity} and Assumption~\ref{assumption:bounded_nonpassivity}.
\par
First, we assume that there are no pedestrians and boundaries in the robot's vicinity and only $\vec{f}_{des}$ drives the robot,
\begin{align}
\sum_{j=1}^N\vec{f}_{soc,j}=\sum_{k=1}^B\vec{f}_{bnd,k}=0.
\end{align}
Therefore, using~\eqref{eq:DynamicsSimple} in Appendix~\ref{sec:app:matrices}, the robot's dynamics~\eqref{eq:DynamicsFinal} becomes
\begin{equation}
\left\{
\begin{aligned}
&m \dot{v} = -\frac{m}{\tau} \bigl(v - v_{des} \cos(\psi)\bigr) + mb \omega^2 \\
&(I + mb^2) \dot{\omega} = -\frac{mbv_{des}}{\tau_d} \sin(\psi) - mbv \omega
\end{aligned}
\right.
\label{eq:system}
\end{equation}
where $\dot{z}=[\dot{v}, \dot{\omega}]^T$, and $-\frac{\pi}{2}\leq\psi\leq \frac{\pi}{2}$ is the angle between the desired velocity $\vec{v}_{des}$ and the robot's velocity $\vec{v}$~($\cos{\psi\ge0}$).
\begin{remark}
    If $\cos{\psi}<0$, the term $(v - v_{des} \cos(\psi))$ in~\eqref{eq:system} changes to $(v + v_{des} \cos(\psi))$ during the global coordinate to local coordinate transformation. 
    Thus, the following discussions apply to the case of $\cos{\psi}<0$, too. 
\end{remark}
Theorem~\ref{theorem:asymptoticallyStable} proves the asymptotic stability of the system~\eqref{eq:system} in the sense of Lyapunov, assuming $v_{des}$ is constant.
First, we show the boundedness of the system for a constant desired speed using Lemmas~\ref{lemma:UUB1} and~\ref{lemma:UUB2}.
Then, the Lemmas support Theorem~\ref{theorem:asymptoticallyStable}.
\begin{lemma}\label{lemma:UUB1}
    Consider an NMR described by~\eqref{eq:system}.
    If $v_{des}$ is constant, the system is Ultimately Uniformly Bounded (UUB).
    Specifically, the speed $v$ enters the bounded region $[0, v_{des}\cos{\psi}]$ after some finite time and remains in that region thereafter.
\end{lemma}
\begin{proof}~ See Appendix~\ref{sec:app:proof}.
\end{proof}
\begin{lemma}\label{lemma:UUB2}
    Consider an NMR described by~\eqref{eq:system}.
    If $v_{des}$ is constant, the system is Ultimately Uniformly Bounded (UUB).
    Specifically, the speed $v$ enters the bounded region $[v_{des}\cos{\psi}, v_{des}]$ after some finite time and remains in that region thereafter.
\end{lemma}
\begin{proof}~ See Appendix~\ref{sec:app:proof}.
\end{proof}
\begin{theorem}\label{theorem:asymptoticallyStable}
    Consider an NMR described by~\eqref{eq:system}.
    If $v_{des}$ is constant, the system is asymptotically stable.
    Specifically, the speed $v$, the angular speed $\omega$, and the relative angle $\psi$ satisfy
    \[
    \lim_{t \to \infty} v(t) = v_{des}, \qquad 
    \lim_{t \to \infty} \omega(t) = 0, \qquad 
    \lim_{t \to \infty} \psi(t) = 0.
    \]
\end{theorem}
\begin{proof}~ See Appendix~\ref{sec:app:proof}.
\end{proof}
Theorem~\ref{theorem:asymptoticallyStable} proves the system's asymptotic stability in the absence of pedestrians and boundaries, i.e., constant desired velocity $\vec{v}_{des}$.
However, in this paper, for practical reasons, $v_{des}$ varies between $v_{min}$ and $v_{max}$ when interacting with other agents in the neighborhood; see~\eqref{eq:v_des}.
If $v_{des}$ is quasi-static or its variation rate is considerably slower than the NMR's states, Theorem~\ref{theorem:asymptoticallyStable} still applies.
Nevertheless, that may not be the case in actual interactions.
Theorem~\ref{theorem:UUBStable} analyzes the system with varying $v_{des}$.
\begin{theorem}\label{theorem:UUBStable}
    Consider an NMR described by~\eqref{eq:system} and bounded but time-varying $v_{des}$.
    The system is UUB, i.e., the speed $v$ and the angular speed $\omega$ eventually enter a bounded region and remain there.
\end{theorem}
\begin{proof}~ See Appendix~\ref{sec:app:proof}.
\end{proof}
Next, we prove that the NMR under the applied desired force $\vec{f}_{des}$, in the presence of pedestrians and boundaries, can be modeled as a port-Hamiltonian network structure with bounded non-passivity (limited activity).
Equivalently, the desired force may inject only a limited amount of energy into the system.
\begin{theorem}\label{theorem:Nonpassivity_f_des}
    The dynamics of the NMR \eqref{eq:system} subjected to the desired force $\vec{f}_{des}$~\eqref{eq:f_des} and with varying $v_{des}$~\eqref{eq:v_des} admit a port-Hamiltonian representation with bounded non-passivity (limited activity).
    Considering $\vec{f}_{des}$ and $B_r^Tz$ as the port's input and output, respectively, 
    \begin{equation}
        \int_0^t z^T B_r \vec{f}_{des}\,d\tau \; \leq \; E_{max}^{des}
    \end{equation}
    for all $t\ge 0$, where $E_{max}^{des}<\infty$ is a finite upper bound on the injected energy.
\end{theorem}
\begin{proof}~Consider the Lyapunov-like function
    \begin{equation}
        V_4=\frac{1}{2}mv^2+\frac{1}{2}(I+mb^2)\omega ^2
    \end{equation}
    with time derivative
    \begin{equation}
    \begin{aligned}
        \dot{V}_4&=-\frac{m}{\tau_d}v(v-v_{des}\cos{\psi})-\frac{mbv_{des}\sin{\psi}}{\tau_d}\omega\\
                 &=z^T B_r \vec{f}_{des}
    \end{aligned}
    \end{equation}
    Integrating both sides gives
    \begin{equation}
    \begin{aligned}
    \int_0^t z^T B_r \vec{f}_{des}\,d\tau = \int_0^t \dot{V}_4 \,d\tau = V_4(t)-V_4(0).
    \end{aligned}
    \end{equation}
    Considering the boundedness of $v$ and $\omega$ according to Theorem~\ref{theorem:UUBStable}, $V_4(t)$ is bounded and there exists a finite $E_{max}^{des}$ which satisfies
    \begin{equation}
        \int_0^t z^T B_r \vec{f}_{des}\,d\tau \; \leq \; E_{max}^{des},
    \end{equation}
    where $E_{max}^{des}=\sup_{t\ge 0} V_4(t)-V_4(0)$.
    Therefore, the port is non-passive and the injected energy is uniformly bounded.
\end{proof}
Similarly, Theorem~\ref{theorem:wall_nonpassivity} proves that the static obstacles/boundaries are boundedly non-passive in the virtual force space.
Analogous to the port-Hamiltonian/network passivity structure, the input and output are $\vec{f}_{bnd,k}$ and $B_r^Tz$, respectively.
Theorem~\ref{theorem:wall_nonpassivity} formalizes bounded non-passivity for all the boundaries and static obstacles in the robot's vicinity and proves limited activity.
\begin{theorem}
\label{theorem:wall_nonpassivity}
    Consider an NMR described by~\eqref{eq:system} and input $u=\vec{f}_{bnd,k}$.
    For all $t\ge0$ and for each boundary or static obstacle, $k=1:B$ in the form of \eqref{eq:f_bnd},
    \begin{equation}
        \int_0^t z^TB_r\vec{f}_{bnd,k}\;d\tau \leq E_{max,k}^{bnd},
        \label{eq:stability_bnd_def}
    \end{equation}
    where $E_{max,k}^{bnd}=\alpha_b\beta_b$.\\
    Hence, the total energy a boundary injects into the system up to any time $t$ is upper-bounded, and the boundaries are boundedly non-passive.
\end{theorem}
\begin{proof}~Let's define $y=\|\vec{d}_k\|$ as the distance from the point of force application $P$ from the static obstacle, or boundary; see Fig.~\ref{fig:DDNMR} and ~\ref{fig:SFMNMR}. 
    In addition, $f_m$ and $f_n$ are the components of $\vec{f}_{bnd,k}(y)$ in the local coordinates.
    Therefore,
    \begin{equation}
        z^TB_r\vec{f}_{bnd,k}=[v\;b\omega]\cdot\begin{bmatrix} f_m \\ f_n \\ \end{bmatrix}=\vec{f}_{bnd,k}\cdot \vec{v}=\dot{y}f_{bnd,k}(y),
        \label{eq:stability_bnd_proof1}
    \end{equation}
    where $\vec{v}$ is the robot velocity vector at point $P$ and $\dot{y}$ is its projection onto $\vec{f}_{bnd,k}$.
    Since $dy=\dot{y}dt$, ~\eqref{eq:stability_bnd_proof1} yields
    \begin{equation}
        \int_0^t z^TB_r\vec{f}_{bnd,k}\;d\tau = \int_0^t\dot{y}f_{bnd,k}\;d\tau=\int_{y_0}^{y_\infty} f_{bnd,k}(y)\;dy,
        \label{eq:stability_bnd_proof2}
    \end{equation}
    with $y_0$ and $y_\infty$ being the distances to the boundary $k$ at $t=0$ and $t=\infty$, respectively.
    Replacing~\eqref{eq:f_bnd} and considering a wider integration range for the always positive and monotonic integral gives
    \begin{equation}
        \int_0^t z^TB_r\vec{f}_{bnd,k}\;d\tau \leq \int_0^\infty\alpha_b e^{-\frac{y}{\beta_b}}\;dy,
        \label{eq:stability_bnd_proof3}
    \end{equation}
    and therefore,
    \begin{equation}
        \int_0^t z^TB_r\vec{f}_{bnd,k}\;d\tau \leq E_{max,k}^{bnd},
        \label{eq:stability_bnd_proof4}
    \end{equation}
    with $E_{max,k}^{bnd}=\alpha_b\beta_b$, which completes the proof.
\end{proof}
Since the pedestrians are not static, we assume that they may be non-passive, but they are not destabilizingly active.
Similar to Theorem~\ref{theorem:wall_nonpassivity}, the port's input and output are $\vec{f}_{soc,j}$ and $B_r^Tz$, respectively.
Assumption~\ref{assumption:bounded_nonpassivity} formalizes the bounded non-passivity concept for all pedestrians in the robot's vicinity.
The assumption is similar to the bounded non-passivity of the leader in teleoperation literature~\cite{Guo2025}.
\begin{assumption}\label{assumption:bounded_nonpassivity}
    There exists a constant $E_{\max,j}^{ped}\ge0$ for each pedestrian $j=1:N$ such that, for all $t\ge0$
    \begin{equation}
        \int_0^t z^TB_r\vec{f}_{soc,j}d\tau \leq E_{\max,j}^{ped}.
        \label{eq:stability_NonPassive_define}
    \end{equation}
    Equivalently, the total energy the pedestrian supplies to the system up to any time $t$ is uniformly bounded.   
\end{assumption}
\begin{theorem}\label{theorem:boundedness}
Consider the NMR dynamics~\eqref{eq:DynamicsFinal} with input described by
\begin{equation}
u=\vec{F}=\sum_{j=1}^{N}\vec{f}_{soc,j}+\sum_{k=1}^{B}\vec{f}_{bnd,k}+\vec{f}_{des}\;,
\label{eq:TotalForce}
\end{equation}
and under Assumption~\ref{assumption:bounded_nonpassivity}.
The Lyapunov function
\begin{equation}\label{eq:stability_V_define}
    V = \tfrac{1}{2} z^T M_r z
\end{equation}
satisfies
\begin{equation}
    V(t) \;\leq\; V(0) + E, \qquad \forall t \geq 0,
\end{equation}
with
\begin{equation}
    E=E_{max}^{des} + \sum_{j=1}^N E_{\max,j}^{ped}+ \sum_{k=1}^B E_{max,k}^{bnd}.
\end{equation}
Consequently, the state $z$ is uniformly bounded,
\begin{equation}
    \|z(t)\|^2 \;\leq\; \frac{2}{\lambda_{\min}(M_r)}\Big(V(0)+E\Big).
\end{equation}
\end{theorem}
\begin{proof}
We select the Lyapunov candidate as~\eqref{eq:stability_V_define}. 
Differentiation and the skew-symmetry property yield
\begin{align}
\dot{V}=z^T(-C_rz+B_r\vec{F})+\frac{1}{2}z^T\dot{M}z=z^TB_r\vec{F},
\label{eq:stability_V_dot_1}
\end{align}
or 
\begin{align}
\dot{V}=z^TB_r\vec{f}_{des}+z^TB_r\sum_{j=1}^N\vec{f}_{soc,j}+z^TB_r\sum_{k=1}^B\vec{f}_{bnd,k}.
\label{eq:stability_V_dot_general}
\end{align}
Integrating~\eqref{eq:stability_V_dot_general} from $0$ to $t$ and using Theorem~\ref{theorem:Nonpassivity_f_des}, Theorem~\ref{theorem:wall_nonpassivity}, and Assumption~\ref{assumption:bounded_nonpassivity} gives
\begin{equation}
    \begin{aligned}
        V(t)-V(0)&=\int_0^tz^TB_r\vec{f}_{des}\\
        &+\int_0^tz^TB_r\sum_{j=1}^N\vec{f}_{soc,j}\;d\tau+\int_0^tz^TB_r\sum_{k=1}^B\vec{f}_{bnd,k}\;d\tau,\\
        &=\int_0^tz^TB_r\vec{f}_{des}+\sum_{j=1}^N\int_0^tz^TB_r\vec{f}_{soc,j}d\tau\\
        &+\sum_{k=1}^B\int_0^tz^TB_r\vec{f}_{bnd,k}d\tau\leq E,
    \end{aligned}
\end{equation}
with
\begin{equation}
    E=E_{max}^{des} + \sum_{j=1}^N E_{\max,j}^{ped}+ \sum_{k=1}^B E_{max,k}^{bnd}.
\end{equation}
Therefore,
\begin{equation}
    V\leq V(0)+E.
    \label{eq:stability_proof1}
\end{equation}
Regarding the second part, since
\begin{equation}
    V=\frac{1}{2}z^TM_rz\geq \frac{1}{2}\lambda_{min}\|z\|^2,
    \label{eq:stability_proof2}
\end{equation}
rearranging~\eqref{eq:stability_proof2} results in
\begin{equation}
     \|z\|^2\leq \frac{2(V(0)+E)}{\lambda_{min}},
\end{equation}
which completes the proof.
\end{proof}
Therefore, the NMR~\eqref{eq:DynamicsFinal} is uniformly bounded with the defined social forces in Section~\ref{sec:main}.
\section{Parameter Calibration}\label{sec:cal}
\textcolor{black}{This section calibrates the models' parameters through simulations.
We sweep a range of selected parameters during hundreds of simulation runs and evaluate a cost function for each trial.}
Table~\ref{tab:simupar} presents the optimized parameters for the models.
\par
Consider a scenario with an NMR and a pedestrian facing each other on a sidewalk with a certain width.
In each trial, NMR and the pedestrian move towards each other and, in a bilateral cooperative interaction, avoid and pass each other, as shown in Fig.~\ref{fig:SFMNMR}.
The NMR moves using SFM with $\Gamma_{SFM}$ in half of the runs and TSFM with $\Gamma_{TSFM}$ in the other half.
The pedestrian movement model is SFM used in~\cite{Liu2022-T-ITS}.\par
\textcolor{black}{Two normalized indices evaluate each parameter set's performance in each trial: comfort metric $I_t$ and speed metric $I_v$.
Pedestrian comfort is the primary goal.
Hence, the final parameter set for each model must realize a comfortable interaction.
Since $T_p$, the minimum of $t_{c,j}$ in each trial, correlated with comfort in pedestrian interaction with e-scooters~\cite{JAFARI2024-2-NATCOM}, $I_t$ is
\begin{align}
I_t=1-\exp{\left(-T_p\right)}.
\label{eq:I_t}
\end{align}
where $I_t=0$ is the most uncomfortable and $I_t=1$ is the ideal outcome.}
\par
\textcolor{black}{However, maximizing $I_t$ is not enough, since there are trivial solutions, for example, when NMR is at a standstill. 
NMR must maximize pedestrian comfort while providing its intended service. 
Therefore, we normalize the robot's speed near the pedestrian, inside an influence zone $\mathcal{I}_z$, to define a secondary metric $I_v$,
\begin{align}
v_{vic}&=\frac{1}{t_z} \int_{\mathcal{I}_z} v \, dt, \label{eq:v_vic}\\
I_v&= \frac{v_{vic}}{v_{max}},
\label{eq:I_v}
\end{align}
where $v$ and $v_{max}$ are the robot's instantaneous and maximum speeds, respectively.
$t_z$ is the duration the robot stays inside $\mathcal{I}_z$.
$I_v=0$ shows a perfect pause, and $I_v=1$ indicates movement with maximum speed all the time.}
\par
Following~\cite{Liu2022-T-ITS}, the influence zone is five meters in front of the pedestrian up to two meters after the robot passes the pedestrian.
The choice is a practical decision.
Selecting a smaller influence zone leads to the minimum speed dominating the metric rather than the average speed. 
This is undesirable as the overall service performance depends on the average speed, not the minimum.
On the other hand, selecting broader zones causes the robots' free motion, unaffected by the pedestrians, to be counted in the metric evaluation.
The presence of free motion at maximum speed in the averaging reduces the impact of the speed drop period on the metric, making it less representative of the interaction. \par
\textcolor{black}{From the SFM and TSFM parameter sets, \eqref{eq:SFMParset} and \eqref{eq:TSFMParset}, the calibration process sweeps the practical ranges of $\Gamma_{SFM}^{cal}=\left[\alpha, \beta, \sigma\right]$ and $\Gamma_{TSFM}^{cal}=\left[T_{cr}, \beta_T, \sigma\right]$, respectively.
The rest of the parameters are set according to Table~\ref{tab:simupar}. 
The sweep looks for the sets of parameters that maximize $I_Q=I_t+I_v$.}
During the search, the algorithm ignores the maximum $ I_Q$ values that occur with very low values of either $I_t$ or $I_v$.
The measure discards extreme cost functions.
They disproportionately favor one metric while neglecting the other and still maximize the summation.
An example of the extreme case is when the robot moves directly to its destination at maximum speed and leaves the avoidance maneuver to the pedestrian.
In summary, the parameter sweep looking for maximum $I_Q$ has constrains on minimum $I_t$ and $I_v$, i.e., $I_v>0.4$ and $I_t>0.65$.
\par
A sample simulation trial is described in Appendix~\ref{sec:app:cal}.
We repeat the trial shown in Fig.~\ref{fig:TSFMsimu} using the parameters shown in Table~\ref{tab:simupar} for both SFM and TSFM and evaluate $I_Q$.
The preset parameters and parameters from the robot dynamics are the same for both models.
Table~\ref{tab:simupar} contains the preset constants, parameters from the dynamics, the results of the calibration process, and the evaluated cost function for the final local maximum.
\par
\textcolor{black}{Since both SFM and TSFM are calibrated to maximize a common cost function, Table~\ref{tab:simupar} fairly compares their performance during the simulations.}
The simulation results indicate that the optimized TSFM performs slightly better in the comfort metric and slightly worse in the speed metric compared to the optimized SFM.
The experiments implement the same set of calibrated parameters for fair comparison.
Section~\ref{sec:res:kin}, specifically Fig.~\ref{fig:boxes}(c) and (d), shows that the gap between the two models is much wider in practice.
\begin{table}
\renewcommand\arraystretch{1.1}
\centering
\caption{SFM and TSFM parameters used in experiments. 
Preset parameters are set based on our previous experience and trial and error.
Parameters from robot dynamics are set to actual values, causing virtual forces that match physical interpretations; see Remark~\ref{remark:ParSel}.
Optimized parameters are the results of parameter sweeps in iterative simulations.
The optimal parameters correspond to the mentioned cost functions.}
\begin{tabular}{||b{0.21\textwidth}|
                >{\centering\arraybackslash}b{0.07\textwidth}|
                >{\centering\arraybackslash}b{0.04\textwidth}|
                >{\centering\arraybackslash}b{0.04\textwidth}||}
\hline
\textbf{Name} & \textbf{Symbol (units)} & \textbf{SFM} & \textbf{TSFM} \\ 
\hline\hline
\multicolumn{4}{||l||}{\textbf{Preset parameters}} \\ 
\hline
Relaxation time & $\tau_d~(s)$ & 0.3 & 0.3 \\ 
\hline
Isotropic constant & $\lambda~(-)$ & 0.04 & -- \\ 
\hline
Boundary force magnitude const. & $\alpha_b~(N)$ & 300 & 300 \\
\hline
Boundary force range const. & $\beta_b~(m)$ & 0.5 & 0.5 \\
\hline
\multicolumn{4}{||l||}{\textbf{Parameters from robot dynamics}} \\
\hline
Mass & $m~(kg)$ & 25 & 25 \\
\hline
Inertia & $I~(kg.m^2)$ & 1.216 & 1.216 \\ 
\hline
Torque arm & $b~(m)$ & 0.5 & 0.5 \\ 
\hline
\multicolumn{4}{||l||}{\textbf{Optimized parameters}} \\
\hline
Force magnitude constant & $\alpha~(N)$ & 125 & -- \\
\hline
Force range constant & $\beta~(m)$ & 1.9 & -- \\ 
\hline
Desired speed adjuster & $\sigma~(N)$ & 140 & 140 \\
\hline
PTTC shift & $T_{cr}~(s)$ & -- & 2.5 \\ 
\hline
Force range constant & $\beta_T~(s)$ & -- & 0.3 \\ 
\hline
\multicolumn{4}{||l||}{\textbf{Optimization cost functions}} \\
\hline
Comfort metric & $I_t~(-)$ & 0.67 & 0.7 \\ 
\hline
Speed metric & $I_v~(-)$ & 0.51 & 0.46 \\ 
\hline
Cost function & $I_Q~(-)$ & 1.18 & 1.16 \\
\hline
\end{tabular}
\label{tab:simupar}
\end{table}
\section{Experiments}\label{sec:exp}
This section explains the experiment setup, design, and procedure.
Section~\ref{sec:exp:setup} details the NMR, the hardware, and the software used in the trials.
Section~\ref{sec:exp:DOE} explains the design of the experiments, the scenario, and the methods.
In addition, for the interested reader, Appendix~\ref{sec:app:Trial} describes a sample trial and extracts the critical variables from the trial.
Moreover, the Institutional Review Board (IRB) approved all experimental procedures.
All participants signed informed consents. \par
\subsection{Experiment setup}\label{sec:exp:setup}
\begin{figure*}
    \centering
    \subfigure[]{
        \includegraphics[width=0.26\textwidth]{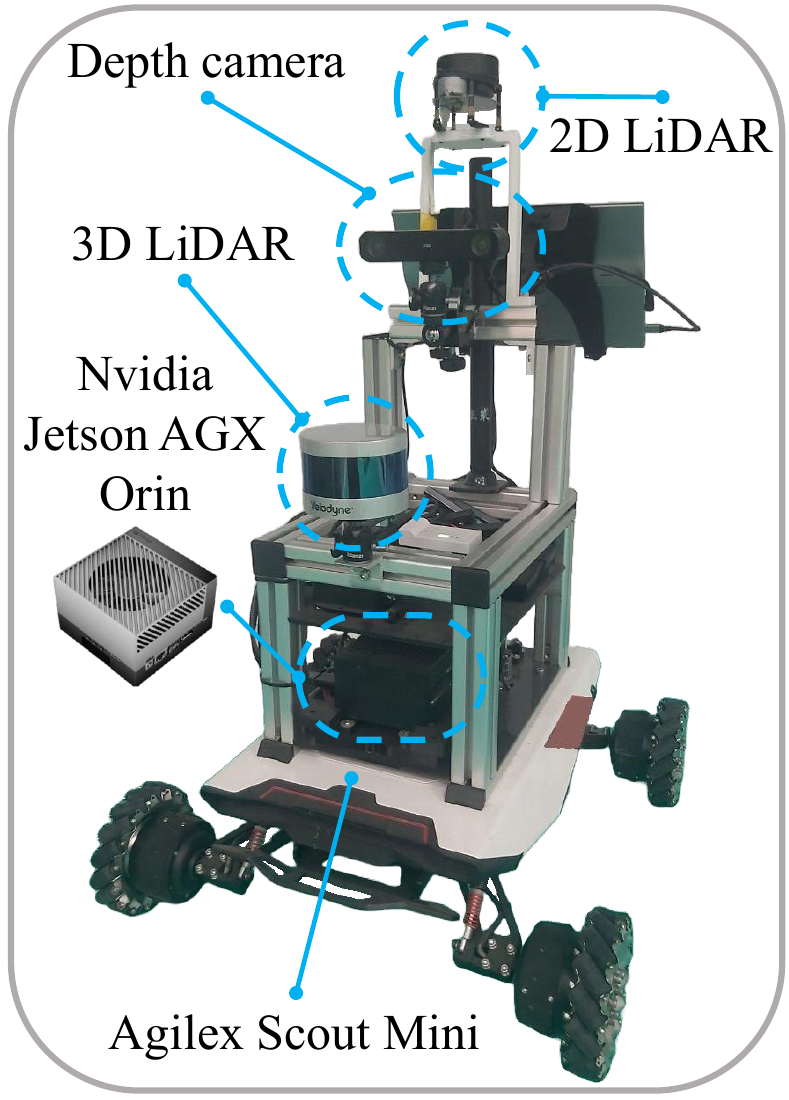}
    }
    \subfigure[]{
        \includegraphics[width=0.51\textwidth]{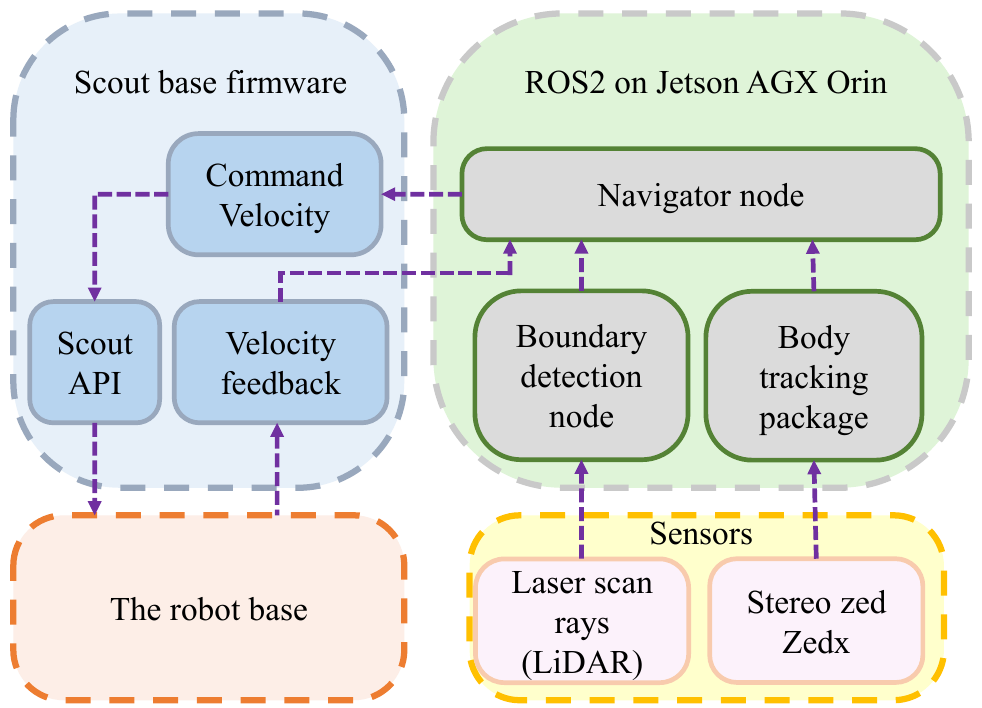}
    }
    \caption{The experiment setup. 
    (a) An Agilex Scout Mini is the mobile base. 
    A Zedx depth camera detects the pedestrians. 
    A 3D LiDAR measures the distance from the walls. 
    An Nvidia Jetson AGX Orin is the processor.
    (b) The block diagram of the whole system includes four major parts:
    The sensors collect the environmental data.
    The ROS2 middleware receives data from the sensors, generates the translational and rotational velocity commands, and sends the commands to the base.
    The base's internal controller realizes the received commands and sends feedback to the navigator node.}
    \label{fig:ExpSetup}
\end{figure*}
\begin{figure}
    \centering
    \subfigure[]{
        \includegraphics[width=0.161\textwidth]{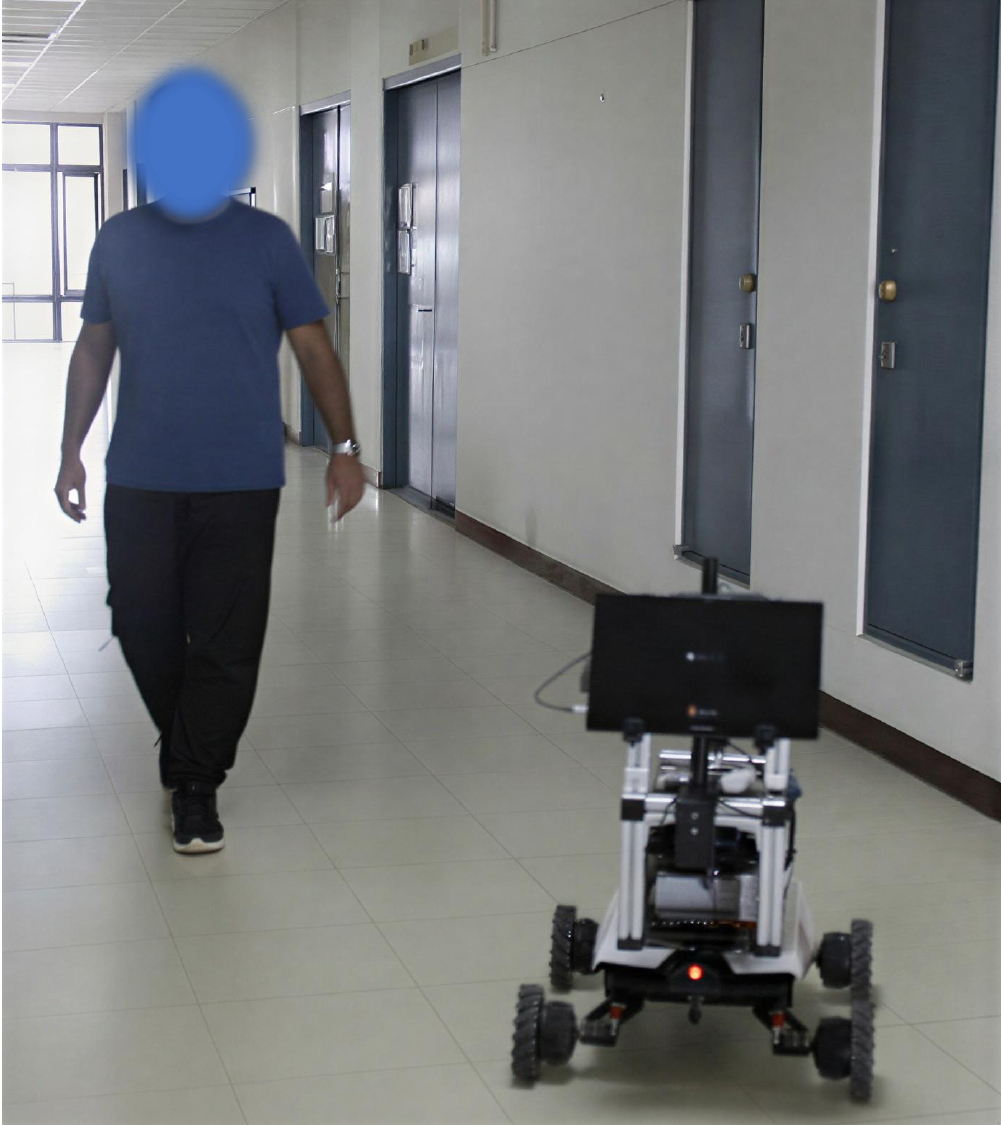}
    }
    \subfigure[]{
        \includegraphics[width=0.29\textwidth]{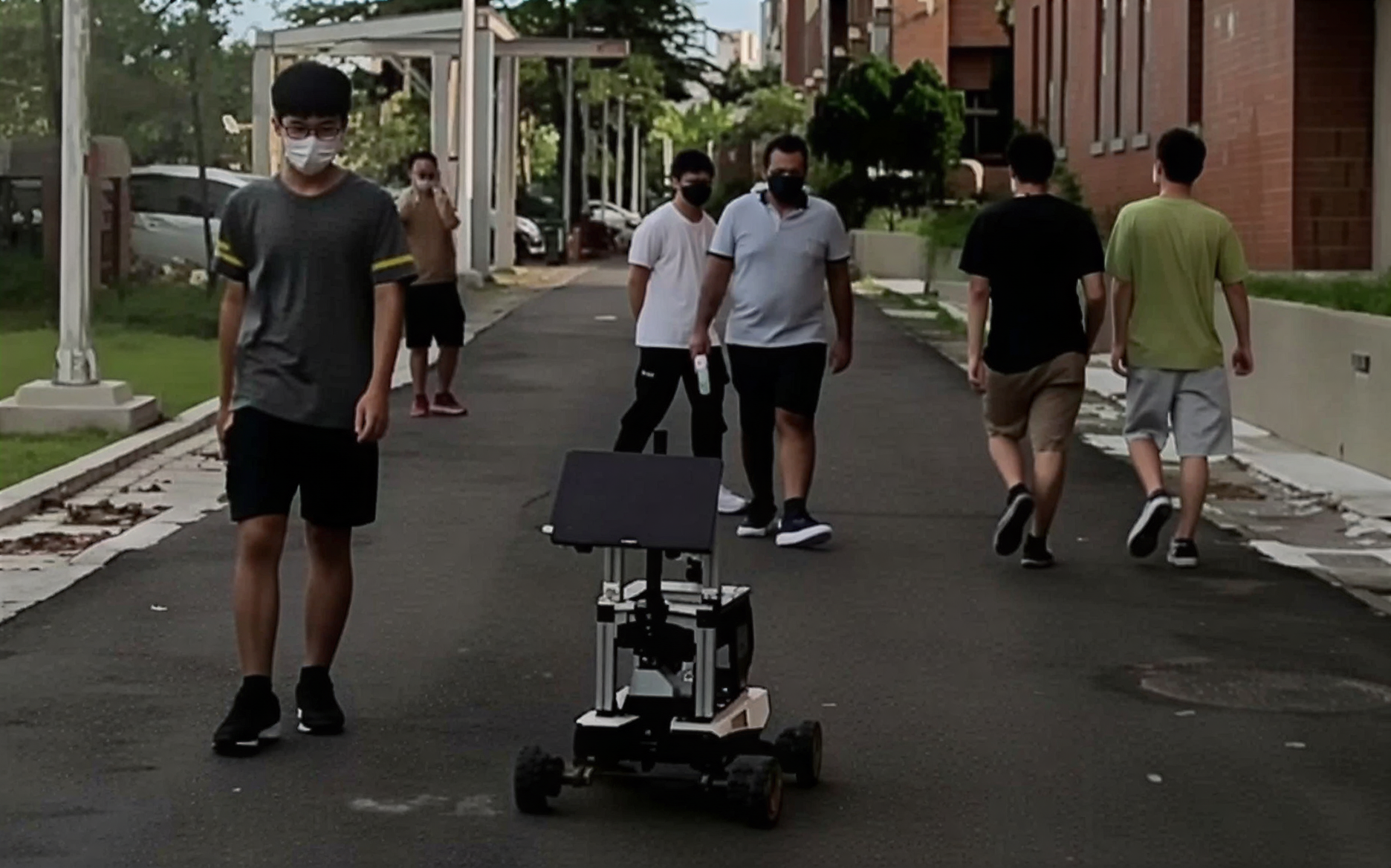}
    }
    \caption{The experiment trials. 
    (a) The robot passes a pedestrian in a sample trial indoors.
    (b) The robot moves through a pedestrian crowd in an outdoor setting.}
    \label{fig:ExpTrial_Vid}
\end{figure}
This section briefly introduces the experiment setup shown in Fig.~\ref{fig:ExpSetup}(a).
We implemented the proposed models and framework on Agilex Scout Mini.
The robot has Mecanum wheels and is holonomic.
We virtually set the lateral velocity to zero, converting it to a nonholonomic differential-drive mobile robot.
A Zedx stereo depth camera detects the pedestrians. 
A Velodyne VLP-16 Puck 3D LiDAR detects the walls.
A 2D LiDAR acts as a redundant collision detector during emergencies.
\par
Fig.~\ref{fig:ExpSetup}(b) is the system block diagram.
The depth camera captures RGB-D frames and uses the zed-ros2-wrapper package to identify the pedestrians, compute the relative positions and velocities, and communicate the topics at a rate of 15 Hz.
Similarly, the 3D Lidar extracts the distance to the walls from the point cloud data and publishes it at a rate of 10 Hz.
An Nvidia Jetson AGX Orin Developer Kit 64GB edge GPU (1.6 GHz) serves as the central processor. 
It collects the data, processes the navigation algorithms, and communicates with the low-level controller of the mobile base. 
The operating system is Ubuntu 22.04 with Jetpack 6.2. We use the ROS2 Humble middleware.
The navigator node itself doesn't require GPU calculations.
Nevertheless, the Jetson CUDA and Tensor cores perform person detection through the collected depth images and boundary identification from LiDAR point clouds.
The central ROS2 node on the Jetson publishes the translational and rotational speed commands to the robot's base at the rate of 50 Hz.
ROS2 $rqt$ package reported the sampling rates.
\subsection{Experiments design}\label{sec:exp:DOE}
This section details the experiment design, demographics, and the participants' biases toward mobile robots.
During all scenarios, NMR and a pedestrian move towards each other and pass in a hallway with a width $3.2~m$.
Pedestrians are instructed to move comfortably as if they are on a sidewalk facing a robot.
Fig.~\ref{fig:ExpTrial_Vid}(a) shows a sample trial with a pedestrian facing an NMR in a hallway. 
In addition to one-on-one trials, we perform trials using SFM and TSFM where the NMR is cruising through a pedestrian crowd outdoors.
Fig.~\ref{fig:ExpTrial_Vid}(b) is a glimpse of the multi-pedestrian trial.
\par
The experiments are performed on the \textcolor{black}{National Cheng Kung University (NCKU) campus}.
Participants are recruited through announcements on the university social media groups.
Of 34 pre-planned volunteers, 22 males and 10 females showed up on the arranged trial time slot.
Participants are 25.7 years old on average, with a standard deviation of 3.58, ranging from 23 to 41 years old.
\par
Due to the volunteer nature of the sampling, a positive bias toward mobile robots is expected.
Among the 32 participants, 20 reported previous familiarity with mobile robots.
To evaluate the bias, the participants responded to a question asking whether they trust mobile robots to operate safely around people on a 1-5 Likert scale before the trial.
They scored 3.84 on average, with a standard deviation of 0.93, indicating a high level of trust in the population compared to neutral populations~\cite{Naneva2020}.
\par
In addition to SFM and TSFM, two remote-controlled scenarios are performed as a baseline for comparison.
During the remote-controlled scenarios, an operator controls the robot using a remote control. 
Similar to the autonomous navigators, the robot moves toward the pedestrian, avoids it, and returns to the center of the hallway after it passes the pedestrian.
The robot speed is set to $1.4~m/s$ and $2.8~m/s$. 
The corresponding trial types are ``R14" and ``R28", respectively.
The participants are told that all interactions are autonomous.
They are unaware of the existence of human-controlled scenarios.
\textcolor{black}{The operators stay at the end of the hall, behind the robot.
They are our lab members and regularly use the mobile robots for their research; hence, highly trained. 
However, they are not involved in this particular study and are unaware of the hypothesis.
They are told to perform the task to the best of their ability.}
\par
\textcolor{black}{In pilot tests without fixing the speed, the human operator frequently reduced the speed.
Considerably lower speed than in the autonomous cases prevents a controlled comparison.
In addition, due to high variability in participants' interpretations, subjective comfort ratings are noisy in slow interactions, whereas at higher speeds the variability decreases.
Thus, we selected moderate (R14) and aggressive (R28) interaction regimes so that the pedestrians' responses would be informative rather than trivial.
Overall, fixing the robot speed in the remote-control cases has two advantages: it creates a human-in-the-loop reference comparable to the autonomous cases and reduces the variability in reported comfort.
However, due to this constraint, the baselines must not be interpreted as human-optimal social navigation.}
\par
\textcolor{black}{We perform the experiments in batches of five trials, a block per method for each participant.
The order of the blocks randomly varies across participants; the volunteers do not know the order.
In other words, different participants attend different sequences of the three methods, but always with five consecutive trials of the same type.}
Thus, the total number of trials for SFM and TSFM is $32\times5=160$ each; all participants attend TSFM and SFM.
In addition, half of the participants attend R14 and the other half R28, resulting in $5\times16=80$ trials each; no participant attends both R14 and R28.
The remote-controlled trial type was arbitrarily selected for each pedestrian.
Of 480 trials, three trials were discarded due to data-saving mishaps: an R14, an R28, and an SFM.
\par
After each trial, the volunteers fill out a questionnaire with four statements and score their opinion regarding each statement on a 1-to-5 Likert scale.
The following statements assess the trial group's effect on the pedestrian's subjective safety:
\begin{itemize}
    \item \textbf{Q1.}~I felt comfortable with how the robot moved as it passed me.
    \item \textbf{Q2.}~The robot’s movement was smooth during the interaction.
    \item \textbf{Q3.}~The distance maintained by the robot while passing me was appropriate
    \item \textbf{Q4.}~The speed of the robot during the interaction felt comfortable to me.
\end{itemize}
\par
Furthermore, we record the robot's speed, heading, pedestrian relative distance, pedestrian relative velocity, and the distances to the left and right walls from the onboard sensors during each trial.
In addition, we save $t_{c,j}$, $\vec{v}_{des}$, $\vec{f}_{des}$, $\vec{f}_{soc,j}$, $\vec{f}_{bnd,k}$, and the command $z$.
Appendix~\ref{sec:app:Trial} shows an SFM and a TSFM trial, further clarifying the models and comparing their forces and speed variations.
\section{Results and Discussions}\label{sec:res}
This section presents the data collected from the trials, statistically compares the trial types, and discusses the implications.
Section~\ref{sec:res:kin} focuses on the interaction kinematics.
Section~\ref{sec:res:survey} statistically studies the pedestrian responses to the survey questions regarding comfort and naturalness of the robot movements across the trial groups.
Then, Section~\ref{sec:res:com} uses metrics developed for evaluating pedestrian comfort and robot sociability to benchmark our proposed algorithms against a recent study.
Section~\ref{sec:res:out} explains the outdoor trials, and Section~\ref{sec:res:lim} concludes the results, reviewing the key findings, the study limitations, and a note regarding the feasibility and practicality of the algorithms.
\par
In Fig.~\ref{fig:boxes}, the data is continuous. 
For the statistical significance tests related to Fig.~\ref{fig:boxes}(a)--(e), we use paired t-tests with the overlapping participants and Welch’s t-test for non-overlapping groups (R14 vs R28).
In Fig.~\ref{fig:Survey_boxes}, the data is ordinal. 
For the statistical significance tests related to Fig.~\ref{fig:Survey_boxes}(a)--(d), we use the Wilcoxon signed-rank test with the overlapping participants and the Mann-Whitney U for non-overlapping groups (R14 vs R28).
\par
\subsection{Statistics of kinematics}\label{sec:res:kin}
\begin{figure*}
\centering
\subfigure[]{\includegraphics[width=0.19\textwidth]{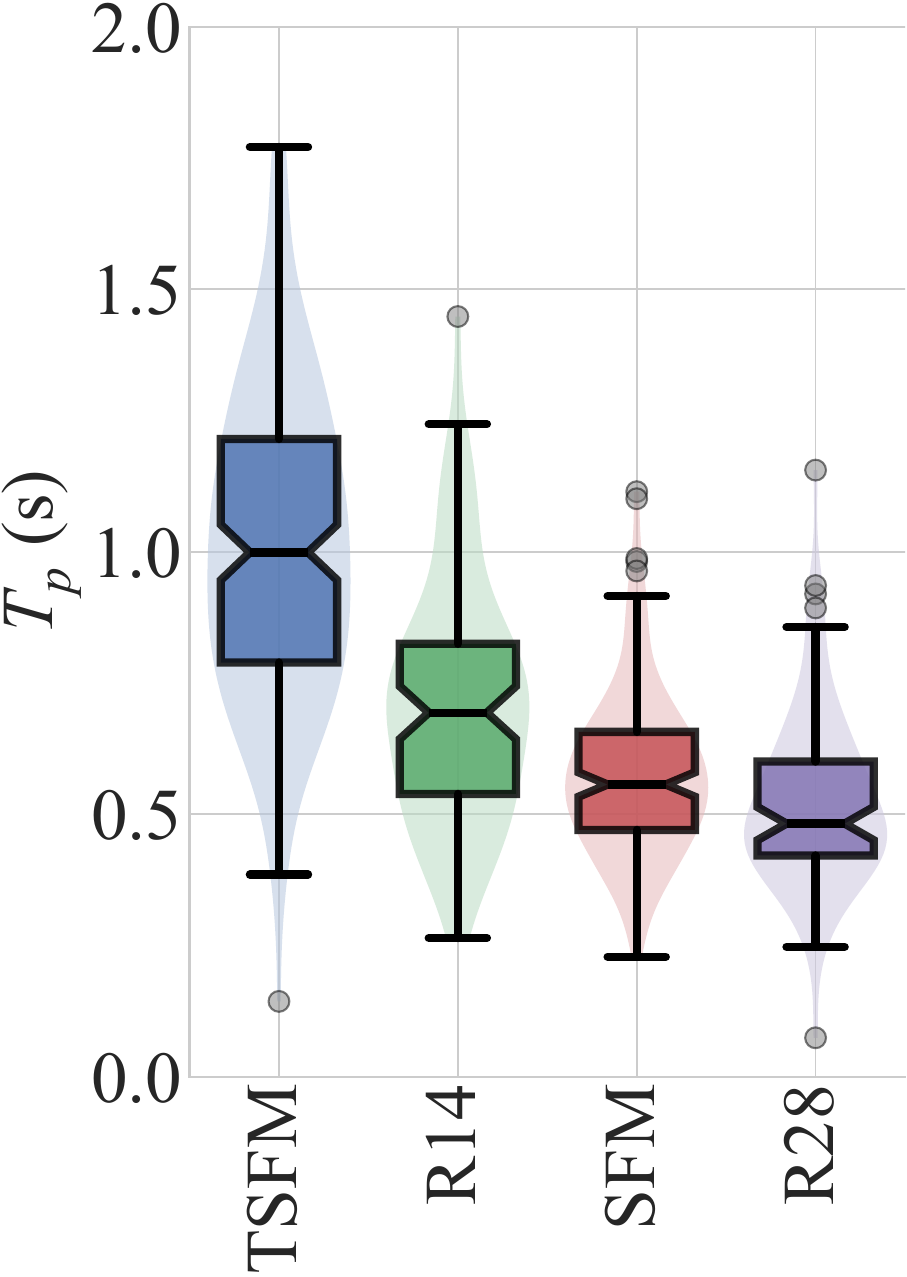}}
\subfigure[]{\includegraphics[width=0.19\textwidth]{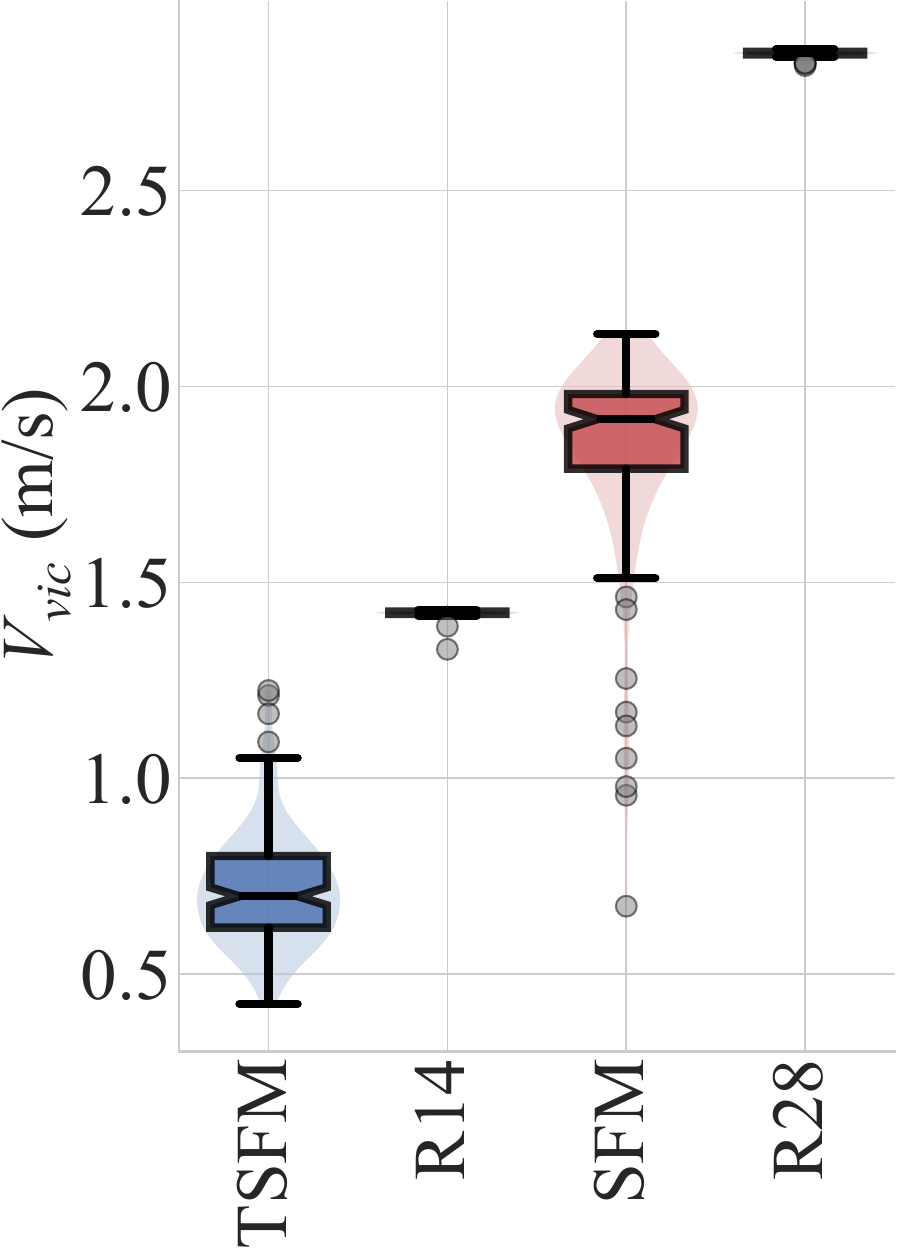}}
\subfigure[]{\includegraphics[width=0.19\textwidth]{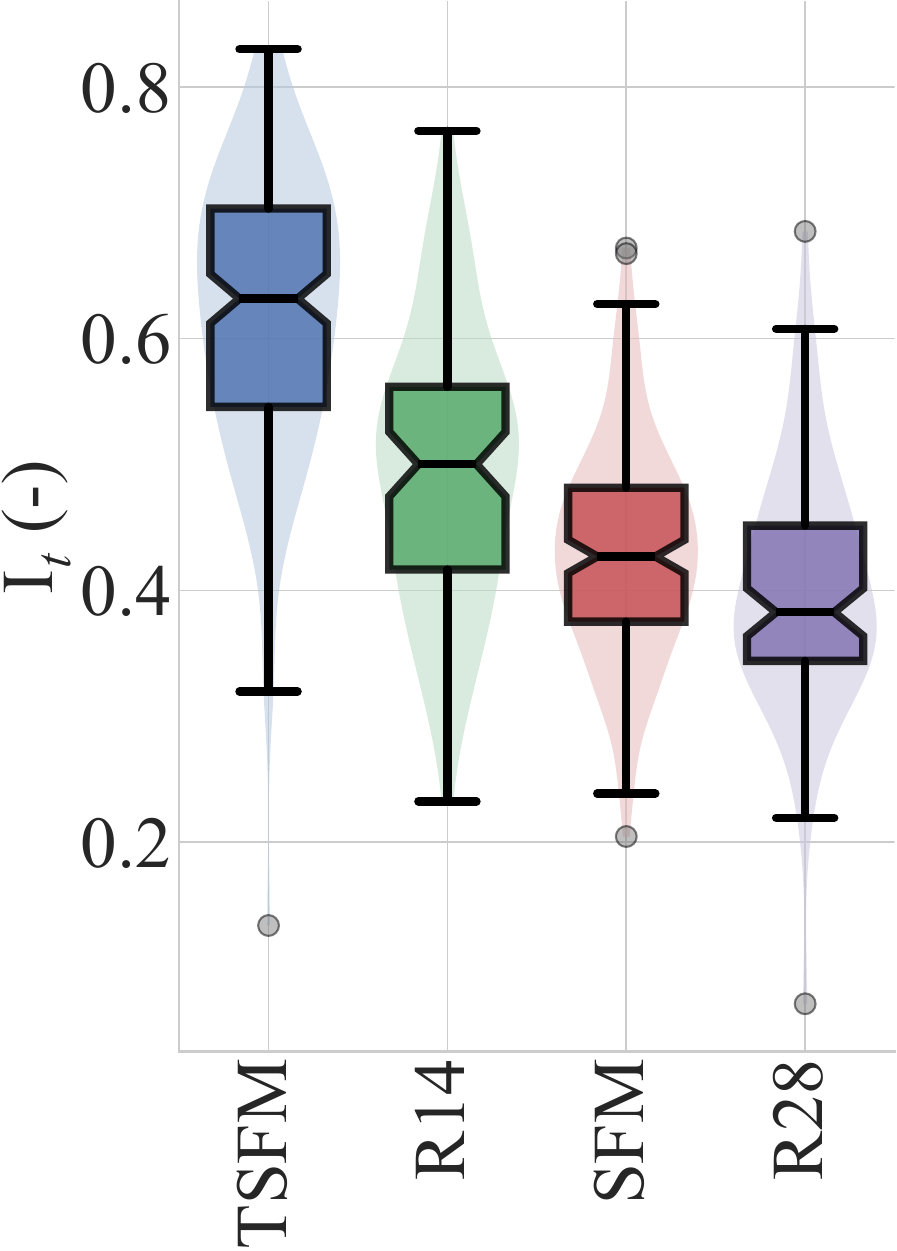}}
\subfigure[]{\includegraphics[width=0.19\textwidth]{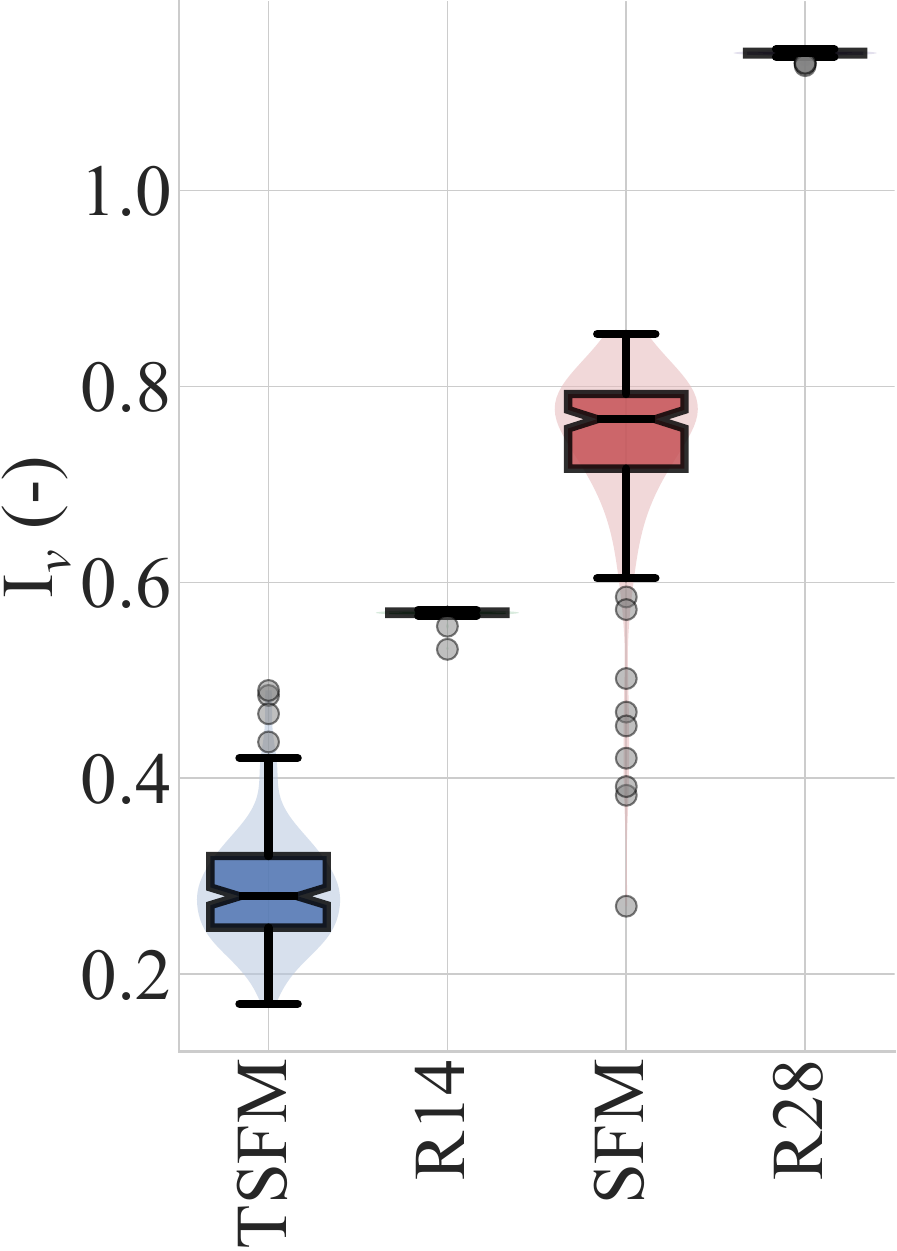}}
\subfigure[]{\includegraphics[width=0.19\textwidth]{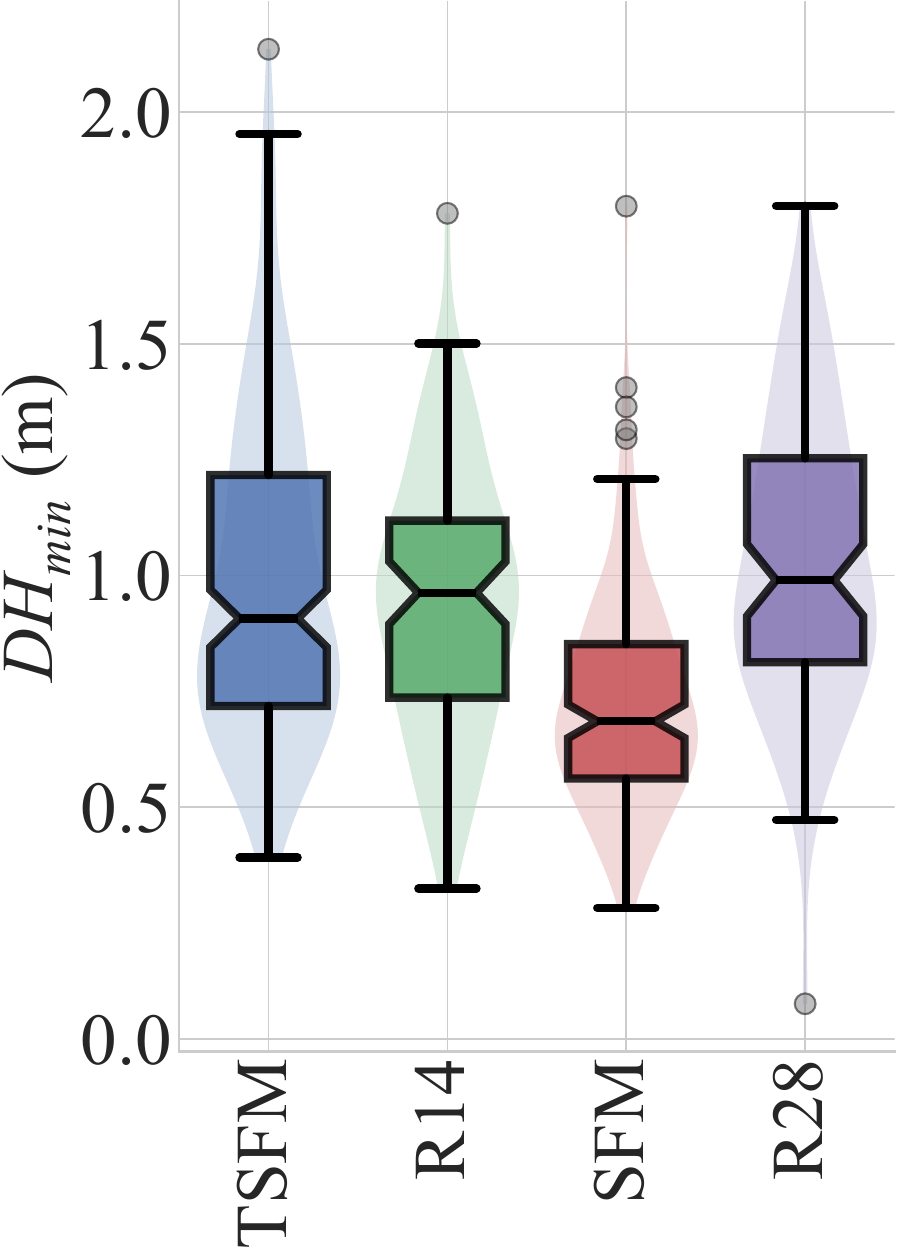}}
\caption{The data extracted from the trajectories recorded during the experimental trials. 
The plot compares four trial types: the robot moves using SFM (Sample size=159), the robot moves using TSFM (Sample size=160), an operator remotely controls the robot moving with a constant speed of $1.4~m/s$ (R14-Sample size=79), and an operator remotely controls the robot moving with a constant speed of $2.8~m/s$ (R28-Sample size=79).
(a) $T_p$: the minimum of PTTC $t_{c,j}$ during the trials.
(b) $v_{vic}$: the average speed in the vicinity of the pedestrian.
(c) $I_t$: the pedestrian comfort metric.
(d) $I_v$: the robot's speed metric.
(e) $DH_{min}$: the minimum distance between the robot and the human during the trials.
}
\label{fig:boxes}
\end{figure*}
In this section, we statistically analyze the experiment results of the autonomous trials against the human-controlled trials.
The kinematics variables benchmarking the trials are the minimum of PTTC $T_p$, the average velocity in the walker's vicinity $v_{vic}$, the comfort metric $I_t$, the speed metric $I_v$, and the minimum distance $DH_{min}$.
Fig.~\ref{fig:boxes} presents the results using standard Tukey box plots with notches as a \textcolor{black}{rough} visual statistical significance test.
\textcolor{black}{Appendix~\ref{sec:app:statistics} Table~\ref{tab:Box_det_Fig5and6} contains the details of Fig.~\ref{fig:boxes}(a)--(e).}
Briefly, TSFM, SFM, R14, and R28 stand for the robot moves autonomously using the SFM model, the robot moves autonomously using the TSFM model, an operator remotely controls the robot with a constant speed of $1.4~m/s$, and an operator remotely controls the robot with a constant speed of $2.8~m/s$, respectively.
\textcolor{black}{In addition, Tables~\ref{tab:St_test_Fig5a}--\ref{tab:St_test_Fig5e} are the results of statistical comparisons for all pairs of trial groups in Fig.~\ref{fig:boxes}.}
\par
Fig.~\ref{fig:boxes}(a),~\textcolor{black}{associated with statistical tests in Table~\ref{tab:St_test_Fig5a}}, compares the methods' $T_p$s. 
TSFM has significantly higher $T_p$ values compared to all other methods and is the most comfortable algorithm among the boxes, supported by a paired t-test on overlapping pedestrians:\\
\textbf{TSFM vs. SFM:~}TSFM ($\text{Mean} = 1.01$, $SD = 0.21$) has higher $T_p$ than SFM ($\text{Mean} = 0.58$, $SD = 0.12$), $t(31) = 12.68$, $p < .005$, $d = 2.24$.\\
\textbf{TSFM vs. R14:~}TSFM ($\text{Mean} = 0.94$, $SD = 0.17$) has higher $T_p$ than R14 ($\text{Mean} = 0.71$, $SD = 0.17$), $t(15) = 4.29$, $p < .005$, $d = 1.07$.\\
\textbf{TSFM vs. R28:~}TSFM ($\text{Mean} = 1.08$, $SD = 0.22$) has higher $T_p$ than R28 ($\text{Mean} = 0.53$, $SD = 0.11$), $t(15) = 10.16$, $p < .005$, $d = 2.54$.
\par
The next comfortable method is R14. 
Since the operator can not change the robot's speed and only steers, its ability to avoid the pedestrian comfortably is limited.
SFM and R28 are the least comfortable algorithms, while SFM performed slightly better.
R28 with a high constant speed creates the most uncomfortable robot-pedestrian interactions. \par
Fig.~\ref{fig:boxes}(b),~\textcolor{black}{associated with statistical tests in Table~\ref{tab:St_test_Fig5b}}, compares $v_{vic}$ during the trials.
TSFM is significantly slower than the other methods. 
Since R14 and R28 have constant speeds, their average speeds don't change in different trials and form very narrow boxes.
SFM average speed in the vicinity of the pedestrian is quite high, causing the significantly lower $T_p$ in Fig.~\ref{fig:boxes}(a).\par
Fig.~\ref{fig:boxes}(c) and Fig.~\ref{fig:boxes}(d),~\textcolor{black}{respectively associated with statistical tests in Tables~\ref{tab:St_test_Fig5c} and \ref{tab:St_test_Fig5d}}, present the normalized comfort~\eqref{eq:I_t} and speed metrics~\eqref{eq:I_v}.
Since the metrics are defined using $T_p$ and $v_{vic}$, they follow the same trends.
Compared to the simulation results in Table~\ref{tab:simupar}, TSFM performs slightly worse on the comfort metric and significantly worse on the speed metric during the experiments.
SFM performs significantly worse in the comfort metric and better in the speed metric.
Perhaps with experimental parameter tuning, SFM's comfort improves at the cost of speed sacrifice and reaches a similar performance to TSFM.
However, we didn't tune the parameters during the experiment for a fair evaluation.
\par
Fig.~\ref{fig:boxes}(e),~\textcolor{black}{associated with statistical tests in Table~\ref{tab:St_test_Fig5e}}, shows the minimum distance between the robot and the pedestrian during the trials.
TSFM maintains about the same minimum distance from the pedestrian.
Unlike TSFM, SFM has a lower minimum distance to the pedestrian.
A paired t-test demonstrates that TSFM is statistically significantly safer than SFM in keeping distance from the pedestrian:\\
\textbf{TSFM vs. SFM:~}TSFM ($\text{Mean} = 0.99$, $SD = 0.26$) has a higher $DH_{\min}$ than SFM ($\text{Mean} = 0.73$, $SD = 0.17$), $t(31) = 7.42$, $p < .005$, $d = 1.31$.
\par
Conventionally, the distance to the pedestrian is considered the main factor in human discomfort~\cite{Hall1963, Neggers2022}.
However, the lack of statistically significant difference hints that the remote-controlled trials have almost the same minimum avoidance distance, despite having significantly different speeds:\\
\textbf{R14 vs. R28 (Welch’s t-test):~}R14 ($\text{Mean} = 0.95$, $SD = 0.20$) does not significantly differ in $DH_{\min}$ from R28 ($\text{Mean} = 1.02$, $SD = 0.22$), $t(29.57) = -1.00$, $p = .328$, $g = -0.35$.\\
In addition, in our experiments, comfort, Fig.~\ref{fig:boxes}(a), does not follow the same pattern as minimum distance, Fig.~\ref{fig:boxes}(e).
The significant comfort difference while keeping the same distance and the different trends in Fig.~\ref{fig:boxes}(a) and Fig.~\ref{fig:boxes}(e) weaken the claim of distance being the main contributor to the comfort.
The result aligns with trends presented by a study focusing on pedestrian comfort estimations~\cite{Jafari2026-1-icra} and highlights the need for further studies to quantify pedestrian comfort.
Overall, TSFM supersedes SFM in both $T_p$ and $DH_{min}$ metrics, indicating that the PTTC integration improves pedestrian comfort.
\par
\subsection{Statistics of subjective safety}\label{sec:res:survey}
\begin{figure*}
\centering
\subfigure[]{\includegraphics[width=0.24\textwidth]{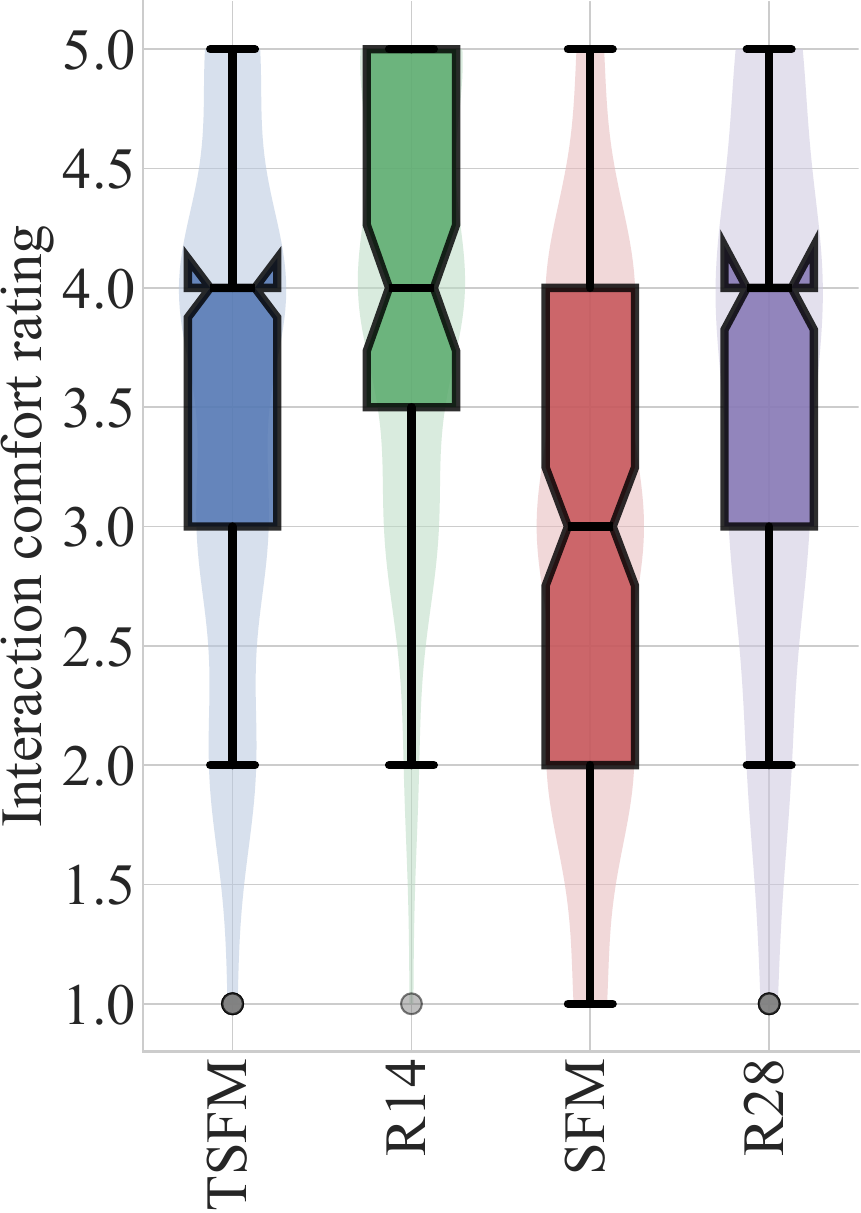}}
\subfigure[]{\includegraphics[width=0.24\textwidth]{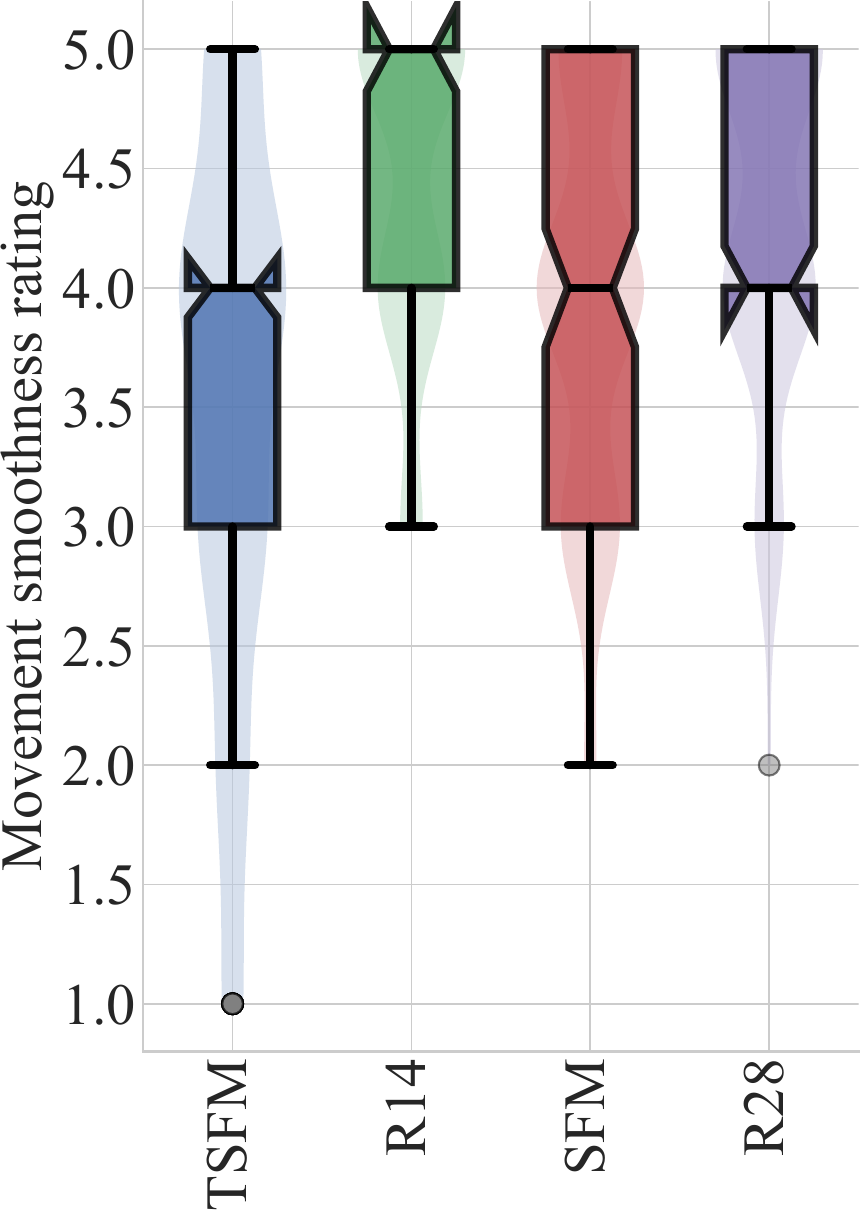}}
\subfigure[]{\includegraphics[width=0.24\textwidth]{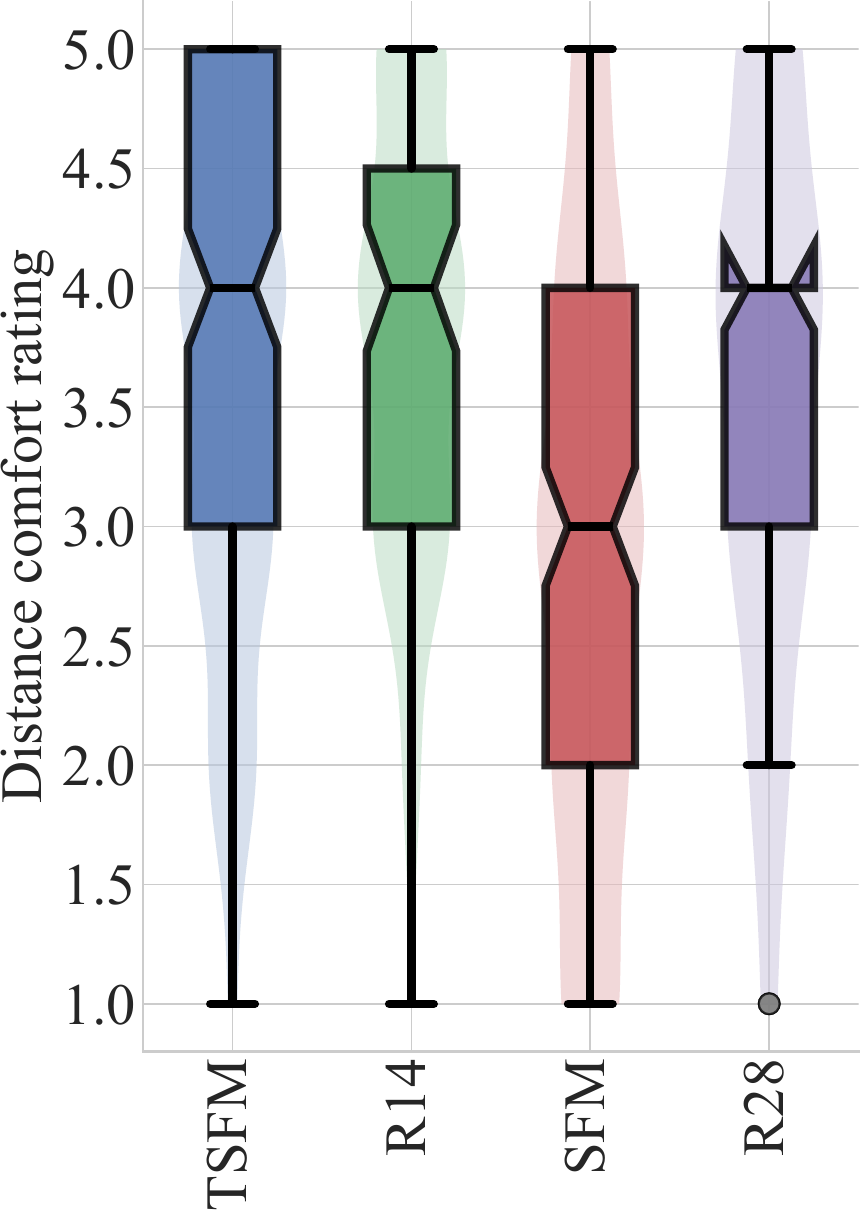}}
\subfigure[]{\includegraphics[width=0.24\textwidth]{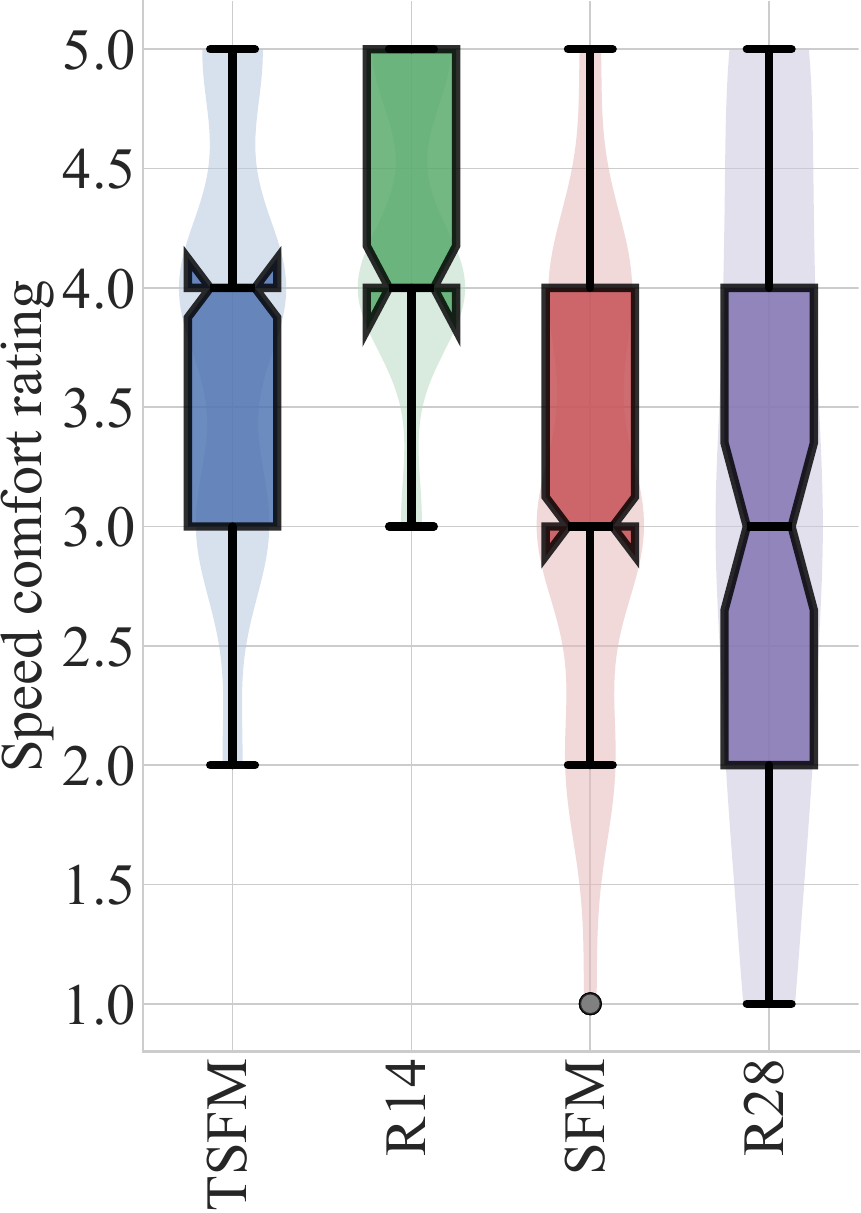}}
\caption{Survey results regarding pedestrians' opinions about the interaction qualitative features on a 1--5 Likert scale.
The plot compares four trial types: the robot moves using SFM (Sample size=159), the robot moves using TSFM (Sample size=160), an operator remotely controls the robot moving with a constant speed of $1.4~m/s$ (R14; Sample size=79), and an operator remotely controls the robot moving with a constant speed of $2.8~m/s$ (R28; Sample size=79).
The statements are 
(a) \textbf{Q1.}~I felt comfortable with how the robot moved as it passed me.
(b) \textbf{Q2.}~The robot's movement was smooth during the interaction.
(c) \textbf{Q3.}~The distance maintained by the robot while passing me was appropriate.
(d) \textbf{Q4.}~The speed of the robot during the interaction felt comfortable to me.
}
\label{fig:Survey_boxes}
\end{figure*}
This section statistically compares the survey responses for TSFM, R14, SFM, and R28, similar to Section~\ref{sec:res:kin}.
The remote-controlled scenarios R14 and R28 are the comparison baselines for the autonomous trials TSFM and SFM; the cases are added due to a lack of comparable literature with a similar setup.
The survey questions benchmarking the trials are about interaction comfort, movement smoothness, distance comfort rating, and speed comfort rating corresponding to Q1--Q4, explained in Section~\ref{sec:exp:DOE}, respectively.
\textcolor{black}{The boxplots in Fig.~\ref{fig:Survey_boxes}(a)--(d) are detailed in Appendix~\ref{sec:app:statistics} Table~\ref{tab:Box_det_Fig5and6}.
Tables~\ref{tab:St_test_Fig6a}--\ref{tab:St_test_Fig6d} show the statistical comparison tests for all pairs of trial groups in Fig.~\ref{fig:Survey_boxes}.}
\par
\par
Fig.~\ref{fig:Survey_boxes}(a),~\textcolor{black}{associated with statistical tests in Table~\ref{tab:St_test_Fig6a}}, presents the main questionnaire take-home result.
TSFM outperforms SFM regarding reported comfort, which is compatible with the results in Fig.~\ref{fig:boxes}(a) and (c):\\
\textbf{TSFM vs. SFM:~}Regarding interaction comfort rating, TSFM ($\text{Median} = 4$, $[Q1,\;Q3] = [3,\;4]$) shows significantly higher comfort than SFM ($\text{Median} = 3$, $[Q1,\;Q3] = [2,\;4]$), $W = 5.50$, $p < .005$, $RBC = .88$.\\
Based on these findings, we deduce that $T_p$ affects pedestrians' comfort.
A distance correlation (dCor) reports a moderate but statistically significant correlation between $T_p$ and the reported comfort in TSFM and R14 trials (TSFM: dCor=0.46 and p=0.001, R14: dCor=0.33 and p=0.005).
The results for SFM and R28 are insignificant for rejecting the null hypothesis.
\par
The correlation is weaker in comparison to the e-scooters, which is in the range of 0.5--0.6~\cite{JAFARI2024-2-NATCOM}.
For e-scooter pedestrian interaction,~\cite{JAFARI2024-2-NATCOM} highlights that the higher the relative velocity between the e-scooter and the pedestrian is, the stronger the correlation between $T_p$ and reported comfort will be.
The e-scooter minimum velocity is about $2.0~m/s$ for a comfortable ride with no balance issues, whereas the maximum set speed of the mobile robot in this study is $2.5~m/s$.
Therefore, a weaker correlation is not unexpected, but the relationship is still considerable and statistically significant.
\par
However, $T_p$ is not the only factor influencing pedestrian subjective safety.
Fig.~\ref{fig:boxes}(a) shows that TSFM has a higher $T_p$ than R14.
Fig.~\ref{fig:boxes}(b) shows that TSFM has a lower relative velocity as well.
In addition, they have almost identical distance and distance comfort ratings in Fig.~\ref{fig:boxes}(e) and Fig.~\ref{fig:Survey_boxes}(c), respectively.
Thus, if these were the only influencing factors on comfort, TSFM should be more pleasant than R14.
Yet, the participants are more comfortable with R14 compared to TSFM.\\
\textbf{TSFM vs. R14:~}Regarding interaction comfort rating, TSFM ($\text{Median} = 4$, $[Q1,\;Q3] = [3,\;4]$) shows significantly lower comfort than R14 ($\text{Median} = 4$, $[Q1,\;Q3] = [4,\;5]$), $W = 9.0$, $p = .05$, $RBC = .80$.
\par
In a similar case, SFM and R28 have almost the same level of $T_p$, but R28, despite being faster, is more comfortable, although it barely misses the $5\%$ threshold:\\
\textbf{SFM vs. R28:~}Regarding interaction comfort rating, SFM ($\text{Median} = 3$, $[Q1,\;Q3] = [2.75,\;4]$) shows a lower comfort than R14 ($\text{Median} = 4$, $[Q1,\;Q3] = [3,\;4]$), $W = 13.5$, $p = .07$, $RBC = .45$.\\
The observations highlight that remote-controlled scenarios have a comfort-related feature other than $T_p$, speed, and minimum distance, and the autonomous algorithms do not fully replicate the features.
\par
A possible reason is the movement's smoothness.
The participants rated both remote-controlled trial types as significantly smoother than the autonomous ones in Fig.~\ref{fig:Survey_boxes}(b),~\textcolor{black}{associated with statistical tests in Table~\ref{tab:St_test_Fig6b}}.\\
\textbf{TSFM vs. R14:~}Regarding interaction smoothness rating, TSFM ($\text{Median} = 4$, $[Q1,\;Q3] = [3,\;4.25]$) shows significantly lower comfort than R14 ($\text{Median} = 5$, $[Q1,\;Q3] = [4,\;5]$), $W = 0.0$, $p < .005$, $RBC = -1.0$.\\
\textbf{SFM vs. R28:~}Regarding interaction smoothness rating, SFM ($\text{Median} = 4$, $[Q1,\;Q3] = [3.75,\;4.25]$) shows significantly lower comfort than R28 ($\text{Median} = 5$, $[Q1,\;Q3] = [4,\;5]$), $W = 0.0$, $p = .007$, $RBC = -1.0$.\\
We believe the movement smoothness, along with $T_p$ and $DH_{min}$, influences the pedestrian comfort, leading to higher comfort ratings for remote-controlled trials in Fig.~\ref{fig:Survey_boxes}(a).
If we shift down the ratings in Fig.~\ref{fig:Survey_boxes}(a) by one point, the comfort pattern matches the pattern for $T_p$ in Fig.~\ref{fig:boxes}(a).
Despite comfort's dependence on smoothness, the features that mark an interaction as smooth for walkers remain unknown and require further research.
A recent study by Brayan et al. reaches the same conclusion~\cite{Brayan2026}.
\par
Fig.~\ref{fig:Survey_boxes}(c),~\textcolor{black}{associated with statistical tests in Table~\ref{tab:St_test_Fig6c}}, is the pedestrian-reported comfort focusing on the maintained distance by the robot.
When compared to Fig.~\ref{fig:boxes}(e), the patterns match, indicating that $DH_{min}$ is an appropriate indicator of the pedestrian understanding.
Furthermore, when we compare the $DH_{min}$ pattern in Fig.~\ref{fig:boxes}(e) with the reported comfort in Fig.~\ref{fig:Survey_boxes}(a), almost the same pattern is visible except for a discrepancy in the R14 group.
Thus, $DH_{min}$ is another metric for walkers' comfort when interacting with a mobile robot.
\par
Fig.~\ref{fig:Survey_boxes}(d),~\textcolor{black}{associated with statistical tests in Table~\ref{tab:St_test_Fig6d}}, shows the pedestrian ratings focused on robot speed.
The pattern aligns with the comfort patterns in Fig.~\ref{fig:Survey_boxes}(a) except for R28.
During the R28 trials, although the robot's speed made the walkers uncomfortable, the overall comfort rating was still relatively high.
We relate the exception to the perceived smoothness during remote-controlled trials.
\par
Overall, using empirical evidence, we summarize that the $T_p$, $DH_{min}$, and smoothness are the driving factors of pedestrian comfort.
Current research mostly uses $DH_{min}$ as the primary metric for pedestrian discomfort and ignores other factors.
A similar setup uses only a remote-controlled robot; the results demonstrate the insufficiency of the minimum distance in quantifying pedestrian comfort~\cite{Jafari2026-1-icra}.
Including these variables refines comfort quantification for mobile robots for a pleasant pedestrian walk.
\subsection{Comparison with previous research}\label{sec:res:com}
Direct comfort comparison with previous research is challenging because research with an empirical comfort study (using surveys) is rare; survey-driven comfort validation is a novelty of this study.
In Sections~\ref{sec:res:kin} and~\ref{sec:res:survey}, we discuss that the minimum distance is neither the only kinematic variable influencing comfort nor the dominant factor.
Nevertheless, in this section, we mainly compare our algorithm with previous research using $DH_{min}$ because it is the standard baseline and the most popular comfort metric.
\par
Table~\ref{tab:MetricComparison} compares our experimental results with a Human-Operated (HO) study~\cite{Brayan2026}, a learning-based method~\cite{Sen2025}, and a game-theory approach~\cite{Sun2026}. 
Brayan et al. present NavWareSet, a dataset that includes an NMR navigation in a crowd~\cite{Brayan2026}; a human operates the robot through a crowd in compliant and noncompliant categories.
In addition, Sen et al. train an NMR to navigate among humans~\cite{Sen2025}.
They use ORCA as their baseline.
Our setup and~\cite{Sen2025} have a similar scenario, mobile robot, and environment; speeds are different.
Moreover, Sun et al. introduce Bayesian Recursive Nash Equilibrium (BRNE) for cooperative collision avoidance~\cite{Sun2026}.
ORCA is also one of their baselines.
We compare the performances of TSFM, SFM, R14, and R28 with HO in compliant and noncompliant categories~\cite{Brayan2026}, with DR and ORCA from~\cite{Sen2025}, and with BRNE and ORCA from~\cite{Sun2026}.
The comparison metrics are~\cite{Sen2025}:
\begin{itemize}
    \item \textbf{Success Rate ($SR$):} Success rate is the number of trials where the robot passes the human without collision and reaches the goal point divided by the total number of trials.
    \item \textbf{Space Compliance ($SC$):} Space compliance is the ratio of the trajectory with less than $1.2~m$ distance to the human to the total trajectory.
    The total trajectory in this paper is the travel length within the influence zone $\mathcal{I}_z$.
    Sen et al. don't explicitly define the SC denominator.
    \item \textbf{Minimum Distance to Human ($DH_{min}$):} It is the minimum relative distance between the NMR and the human during each trial, i.e., Fig.~\ref{fig:boxes}(e).
\end{itemize}
Table~\ref{tab:MetricComparison} presents the means of the collected data.
The comparison is based on reported results by~\cite{Brayan2026, Sen2025, Sun2026} and does not use a unified experimental framework.
\par
\begin{table}
\centering
\renewcommand\arraystretch{1.2}
\caption{Robots' performance comparison between six autonomous navigation methods and four remote-controlled cases using mean values of metrics. 
HOs are human-operated trials in compliant and noncompliant cases for the frontal approach scenario using an NMR~\cite{Brayan2026};
R14 and R28 are remote-controlled cases at $1.4~m/s$ and $2.8~m/s$ constant speeds;
ORCA trials are baselines in~\cite{Sun2026} and~\cite{Sen2025};
DR uses the domain randomization technique~\cite{Sen2025};
BRNE uses a game theory approach~\cite{Sun2026};
SFM is our baseline;
TSFM uses PTTC to improve pedestrians' comfort.
}
\label{tab:MetricComparison}
\begin{tabular}{||m{0.18\textwidth}|p{0.06\textwidth}|p{0.06\textwidth}|p{0.08\textwidth}||}
\hline
\diagbox[width=13.12em, height=3.7em]{\textbf{Method}}{\textbf{Metric}} & \textbf{$SR~(\%)$} & \textbf{$SC~(\%)$} & \textbf{$DH_{min}~(m)$} \\
\hline\hline
\textbf{HO-compliant}~\cite{Brayan2026} & $100$ & $-$ & $1.06$ \\
\hline
\textbf{HO-noncompliant}~\cite{Brayan2026} & $100$ & $-$ & $0.56$ \\
\hline
\textbf{R14 (Ours)} & $100$ & $11$ & $0.95$ \\
\hline
\textbf{R28 (Ours)} & $100$ & $10$ & $1.02$ \\
\hline
\textbf{ORCA}\cite{Sen2025} & $-$ & $49$ & $0.48$ \\
\hline
\textbf{ORCA}~\cite{Sun2026} & $73$ & $-$ & $0.6^{\mathrm{a}}$ \\
\hline
\textbf{DR}\cite{Sen2025} & $91^{\mathrm{a}}$ & $33$ & $0.59$ \\
\hline
\textbf{BRNE}~\cite{Sun2026} & $97$ & $-$ & $1.2^{\mathrm{a}}$ \\
\hline
\textbf{SFM (Ours)} & $100$ & $14$ & $0.73$ \\
\hline
\textbf{TSFM (Ours)} & $100$ & $11$ & $0.99$ \\
\hline
\end{tabular}
\vspace{1ex}
\\
\raggedright
\text{$^{a}$ Simulations results replace unavailable experimental results.}
\end{table}
The NMR using SFM and TSFM successfully avoided the pedestrian and passed him/her toward the end of the hallway in all $320$ trials, compared to DR's 91\% and ORCA's 73\% success rates.
In addition, a lower portion of the trajectory is inside the pedestrian intimate zone~\cite{Hall1963}, reflected in lower $SC$s.
Regarding space compliance, TSFM outperforms all other autonomous methods, including SFM.
In addition, TSFM has a similar SC to our remote-controlled cases.
\par
The minimum relative distance to human $DH_{min}$ indicates better performance for SFM and TSFM compared to DR, both ORCAs, and noncompliant HO.
TSFM stays further from the pedestrian compared to SFM, DR, both ORCAs, and noncompliant HO; TSFM's performance is similar to our remote-controlled scenarios and the compliant HO.
However, BRNE outperforms TSFM in the avoidance distance metric by about 20\% margin.
We highlight that SFM and TSFM are explicit models, don't need training, and are simpler and more interpretable compared to other methods.
\par
\subsection{Outdoor multi-pedestrian trials}\label{sec:res:out}
We conduct a qualitative proof of concept similar to Huber et al.~\cite{Huber2022}.
Outdoor trials with multiple pedestrians on an actual sidewalk verify that the methods work in real-world situations; see Fig.~\ref{fig:ExpTrial_Vid}(b) and the Supplementary video. 
\par
The NMR must move forward in a predefined direction and pass through multiple moving pedestrians with arbitrary set directions. 
Some pedestrians move together as a group, and some others walk individually.
The pedestrians' speed is within a normal walking range, and they are free to change directions or avoid the robot to resemble a delivery robot in service.
The trial ends when the NMR passes the crowd and centralizes on the sidewalk.
\par
\textcolor{black}{
We designed TSFM to react to PTTC explicitly; the robot's reaction to low PTTCs exponentially grows by design, i.e.,~\eqref{eq:f_soc:TSFM}--\eqref{eq:f_soc:TSFMeq}.
In the crowd section of the supplementary video, the TSFM robot's exponential reaction to low PTTC is visible. 
In the TSFM trial, the robot stops, slightly moves backward, and readjusts before moving forward, whereas in the SFM trial, the robot deviates and keeps moving as long as the distance is acceptable, even if PTTC is low.
TSFM prioritizes maintaining safer temporal margins, whereas SFM prioritizes spatial separation.
Because TSFM maintains a higher PTTC and since PTTC's minimum $T_p$ and comfort are correlated, TSFM is statistically more comfortable than SFM.
}
\par
\textcolor{black}{Overall,} using both SFM and TSFM, the NMR successfully navigates through pedestrian crowds on unstructured sidewalks where the robot simultaneously performs head-on encounters, overtaking maneuvers, and other complex multi-agent interactions without requiring extensive scenario-specific training.
The outdoor tests show that the methods work beyond the controlled environments and support the practical feasibility of applying SFM and TSFM to NMRs.
\par
\subsection{Highlights and limitations}\label{sec:res:lim}
This paper suggests a framework for NMR-pedestrian interactions.
To our knowledge, no research has studied the stability of a mobile robot when a pedestrian is involved.
Using a passivity-based approach, we prove the stability of our algorithms applied to the system.
More importantly, we formalize the agents' dynamics on sidewalks, establishing a framework for further algorithm developments.
\par
Regarding the proposed algorithms, we benchmark pedestrian comfort using the introduced metrics.
TSFM outperforms SFM in pedestrian comfort, but the advantage comes at the cost of robot speed.
In addition, TSFM's movement patterns show similar minimum distance and space compliance to the remote-controlled trials.
Both TSFM and SFM surpass ORCA and DR in the metrics, even though the robot moves significantly faster.
\par
Regarding reported subjective safety, $T_p$ moderately correlates with comfort, providing a viable option for incorporating pedestrians' feelings into navigation algorithms using on-board sensors.
Nevertheless, further analysis shows that movement smoothness is also a major contributor to subjective safety and suggests future directions for improving walkers' experience on futuristic sidewalks.
\par
Versatility in handling various scenarios without training and simplicity are additional advantages of the proposed methods.
The simplicity of SFM and TSFM is another advantage over learning methods.
They do not require extensive scenario-dependent training, use explicit formulations, are intuitive, and are easy to debug and further develop.
\par
The study faces a few methodological constraints.
Although being comparable to similar research in volunteer group size, a larger population that is more gender and age-balanced improves the reliability of the results.
Moreover, the positive bias in attitude toward mobile robots distorts the questionnaire response distributions.
The post-trial collection of the pedestrian opinions introduces noise into the reported comfort, too.
Another limitation is using only two preset speeds in the remote-controlled case, which may not fully reveal the trends.
Therefore, we recommend using several intermediate speed steps.
\section{Conclusion}\label{sec:con}
Futuristic sidewalks include mobile robots cruising through pedestrian crowds.
Due to practical necessities, most of these robots are nonholonomic.
However, nearly all research on mobile robot-pedestrian interaction in public spaces is developed for holonomic mobile robots.
We propose a framework for force-based navigation algorithms for application to NMRs.
\par
As proof of concept, we apply SFM and TSFM in the framework.
Since PTTC correlates with pedestrian discomfort, we extend a PTTC-based SFM (TSFM) from e-scooters to NMRs.
We develop, calibrate, and experimentally compare the introduced models in terms of pedestrian comfort and the robot's speed.
The stability analysis of mobile robot-pedestrian interaction using a passivity-based approach is another novelty.
Moreover, outdoor multi-pedestrian tests verify the algorithms' performance on real sidewalks.
\par
Furthermore, we compare the proposed methods to a previous learning-based approach using criteria introduced by the research.
The results favor the SFM-based approaches in terms of pedestrian comfort metrics.
Moreover, statistical analysis of the survey results shows that TSFM outperforms SFM in walkers' comfort.
In addition, the autonomous algorithms match the remote-controlled cases in $T_p$ and space compliance.
However, the human-driven robots move more smoothly than the autonomous navigators, yielding partial improvement of walkers' comfort.
Overall, we recommend TSFM and SFM for NMRs due to their practical feasibility, implementation simplicity, and superior performance.
\par
\textcolor{black}{This paper does not include any predictions of pedestrian movements.
However, incorporating pedestrian path predictions is a promising research direction.
Predicting walkers' trajectories leads to smoother robot motion, higher PTTC, and more comfortable interactions.
High-level predictions address complex cases, such as sudden stops and sharp turns in crowds, deadlocks, pushovers, and online detection and reaction to pedestrian social group behaviors—challenges hardly manageable by methods relying only on instantaneous states.}
Extending the framework to a cooperating group of mobile robots is another challenge worth considering.
Our ongoing research focuses on quantifying the pedestrians' comfort using questionnaires in similar scenarios.

\ifCLASSOPTIONcaptionsoff
  \newpage
\fi

	\bibliographystyle{IEEEtran}
	\bibliography{Ref}
\appendices
\section{Detailed dynamics}\label{sec:app:matrices}
The matrices introduced in \eqref{eq:Dynamics} are
\begin{equation}
    \label{MCB}
    M= \begin{bmatrix}
        m & 0 & -mbS_{\theta} \\
        0 & m & mbC_{\theta} \\
        -mbS_{\theta} & mbC_{\theta} & I + mb^2
    \end{bmatrix},
\end{equation}
\begin{equation*}
    C= \begin{bmatrix}
        0 & 0 & -mb\dot{\theta}C_{\theta} \\
        0 & 0 & -mb\dot{\theta}S_{\theta} \\
        0 & 0 & 0
    \end{bmatrix}\text{, and~}
    B= \begin{bmatrix}
    C_{\theta} & -S_{\theta} & 0\\
    S_{\theta} & C_{\theta}  & 0 \\
    0 & b & 1
    \end{bmatrix}.
\end{equation*}
The matrices introduced in \eqref{eq:DynamicsReduced} are
\begin{equation}
    \label{MCB_r}
    M_r = \begin{bmatrix}
        m & 0 \\
        0 & I + mb^2
    \end{bmatrix},
    C_r = \begin{bmatrix}
        0 & -mb\dot{\theta}\\
        mb\dot{\theta} & 0 \\ 
    \end{bmatrix},
\end{equation}
    \begin{equation*}
    B_r = \begin{bmatrix}
        1 & 0 & 0\\
        0 & b & 1
    \end{bmatrix}\text{, and~}
    \dot{J}^T = \begin{bmatrix}
        \dot{\theta}C_{\theta} & -\dot{\theta}S_{\theta} & 0 \\
        0 & 0 & 0
    \end{bmatrix}. 
\end{equation*}
In addition, \eqref{eq:DynamicsFinal} is equivalent to
\begin{equation}
\left\{
\begin{aligned}
    & m\dot{v}-mb\omega^2 = F_m\\
    & (I+mb^2)\dot{\omega}+mb\omega v = bF_n+\tau
\end{aligned}
\right.,
\label{eq:DynamicsSimple}
\end{equation}
where $F_m$ and $F_n$ are $\vec{F}$ components in the local frame.
\section{Proofs}\label{sec:app:proof}
\textbf{Lemma 1.}
    Consider an NMR described by~\eqref{eq:system}.
    If $v_{des}$ is constant, the system is Ultimately Uniformly Bounded (UUB).
    Specifically, the speed $v$ enters the bounded region $[0, v_{des}\cos{\psi}]$ after some finite time and remains in that region thereafter.
\begin{proof}~ Consider the Lyapunov candidate
\begin{align}
    V_1=\frac{1}{2}mv^2+\frac{1}{2}(I+mb^2)\omega^2+\frac{mbv_{des}}{\tau_d}(1-\cos{\psi})\ge0.
    \label{eq:V1}
\end{align}
Differentiating $V_1$ along the system trajectories gives
\begin{equation}
\begin{aligned}
\dot{V}_1 &= m\dot{v}v + (I+mb^2)\dot{\omega}\omega + \frac{mbv_{des}}{\tau_d}\omega \sin{\psi} \\
&= -\frac{m}{\tau_d}v(v-v_{des}\cos{\psi}) + mb\omega^2 v \\
&\quad -\frac{mbv_{des}}{\tau_d}\omega \sin{\psi} - mbv\omega^2 + \frac{mbv_{des}}{\tau_d}\omega \sin{\psi} \\
&= -\frac{m}{\tau_d}v(v-v_{des}\cos{\psi}).
\end{aligned}
\label{eq:V_dot_1}
\end{equation}
$\dot{V}_1$ is negative for all $v$ outside the region $[0, v_{des}\cos{\psi}]$. 
If $v$ is initially outside this region, it approaches the region and enters it after some finite time. 
Once $v$ is inside the region, it remains there since once it exists, $\dot{V}_1$ becomes negative and drives it back inside.
Thus, the system is Ultimately Uniformly Bounded and after some finite time $0\leq v\leq v_{des}\cos{\psi}$.
\end{proof}
\textbf{Lemma 2.}
    Consider an NMR described by~\eqref{eq:system}.
    If $v_{des}$ is constant, the system is Ultimately Uniformly Bounded (UUB).
    Specifically, the speed $v$ enters the bounded region $[v_{des}\cos{\psi}, v_{des}]$ after some finite time and remains in that region thereafter.
\begin{proof} 
~ Consider the Lyapunov candidate
\begin{equation}
    \begin{aligned}
    V_2&=\frac{1}{2}m(v-v_{des})^2+\frac{1}{2}(I+mb^2)\omega^2\\
    &+\frac{mbv_{des}}{\tau_d}(1-\cos{\psi})\ge0.
    \end{aligned}
    \label{eq:V2}
\end{equation}
Differentiating $V_2$ along the system trajectories gives
\begin{equation}
\begin{aligned}
\dot{V}_2 &= m\dot{v}(v-v_{des}) + (I+mb^2)\dot{\omega}\omega + \frac{mbv_{des}}{\tau_d}\omega \sin{\psi} \\
&= -\frac{m}{\tau_d}(v-v_{des})(v-v_{des}\cos{\psi}) + mb\omega^2 v - mb\omega^2 v_{des}\\
&\quad -\frac{mbv_{des}}{\tau_d}\omega \sin{\psi} - mbv\omega^2 + \frac{mbv_{des}}{\tau_d}\omega \sin{\psi} \\
&= -\frac{m}{\tau_d}v(v-v_{des}\cos{\psi})- mb\omega^2 v_{des}.
\end{aligned}
\label{eq:V_dot_2}
\end{equation}
$\dot{V}_2$ is negative for all $v$ outside the region $[v_{des}\cos{\psi}, v_{des}]$. 
If $v$ is initially outside this region, it approaches the region and enters it after some finite time. 
Once $v$ is inside the region, it remains there since once it exists $\dot{V}_2$ becomes negative and drives it back inside.
Thus, the system is Ultimately Uniformly Bounded and after some finite time $v_{des}\cos{\psi}\leq v\leq v_{des}$.
\end{proof}
\par
\textbf{Theorem 1.}
    Consider an NMR described by~\eqref{eq:system}.
    If $v_{des}$ is constant, the system is asymptotically stable.
    Specifically, the speed $v$, the angular speed $\omega$, and the relative angle $\psi$ satisfy
    \[
    \lim_{t \to \infty} v(t) = v_{des}, \qquad 
    \lim_{t \to \infty} \omega(t) = 0, \qquad 
    \lim_{t \to \infty} \psi(t) = 0.
    \]
\begin{proof}~
The proof directly follows Lemmas~\ref{lemma:UUB1} and~\ref{lemma:UUB2}. 
The speed $v$ is UUB in two regions that share a single point, $v_{des}\cos{\psi}$; therefore, $\lim_{t \to \infty} v(t)=v_{des}\cos{\psi}$.
\par
Considering \eqref{eq:V_dot_2}, $\dot{V}_2\leq 0$ when $\lim_{t \to \infty} v(t)=v_{des}\cos{\psi}$.
Therefore, the system is stable in the sense of Lyapunov.
Moreover, by LaSalle’s invariance principle, as $t \to \infty$, every trajectory approaches the largest, and only in this case, invariant set where $V_2=0$.
Thus, $\omega$ and $\psi$ converge to the origin. Now, we can conclude that as ${t \to \infty}$, $v(t) \to v_{des}$, which completes the proof.
\end{proof}
\textbf{Theorem 2.}
    Consider an NMR described by~\eqref{eq:system} and bounded but time-varying $v_{des}$.
    The system is UUB, i.e., the speed $v$ and the angular speed $\omega$ eventually enter a bounded region and remain there.
\begin{proof}~Consider the partial Lyapunov candidate
\begin{align}
    V_3=\frac{1}{2}m(v-v_{max})^2+\frac{1}{2}(I+mb^2)\omega^2\ge0.
    \label{eq:V3}
\end{align}
Differentiation yields
\begin{align}
    \dot{V}_3&=V_{3d}^v+V_{3d}^\omega,
    \label{eq:V_dot_3}
\end{align}
where
\begin{align}
    V_{3d}^v&=-\frac{m}{\tau_d}(v-v_{max})(v-v_{des}\cos{\psi})~\text{and}\\
    V_{3d}^\omega&=-mbv_{max}\omega^2-\frac{mbv_{des}}{\tau_d}\omega\sin{\psi}.
    \label{eq:V3d}
\end{align}
Noting that $v_{des}\leq v_{max}$ and $\lvert \cos{\psi} \rvert\leq 1$, it follows that for $\lvert v \rvert> v_{max}$, $V_{3d}^v<0$.
Similarly, for $\lvert \omega \rvert> \frac{1}{\tau_d}$, we have $V_{3d}^\omega<0$.
Hence, outside the regions $\lvert v \rvert < v_{\max}$ and $\lvert \omega \rvert < 1/\tau_d$, $\dot{V}_3$ is negative, and therefore, $v$ and $\omega$ can not diverge simultaneously.
\par
On the other hand, $V_{3d}^v$ and $V_{3d}^\omega$ are in quadratic form, and thus, upper bounded inside the regions, i.e., between their roots.
As a result, the divergence of neither of the terms can be compensated by the other term to keep $\dot{V}_3\ge0$; $\dot{V}_3$ becomes negative if either of $v$ or $\omega$ diverges.
Consequently, neither of them diverges, and both $v$ and $\omega$ eventually enter a bounded region and stay there.
Therefore, the system is UUB, and the proof is complete.
\end{proof}
\section{A calibration run}\label{sec:app:cal}
\begin{figure}
\centering
\subfigure[]{\includegraphics[width=0.48\textwidth]{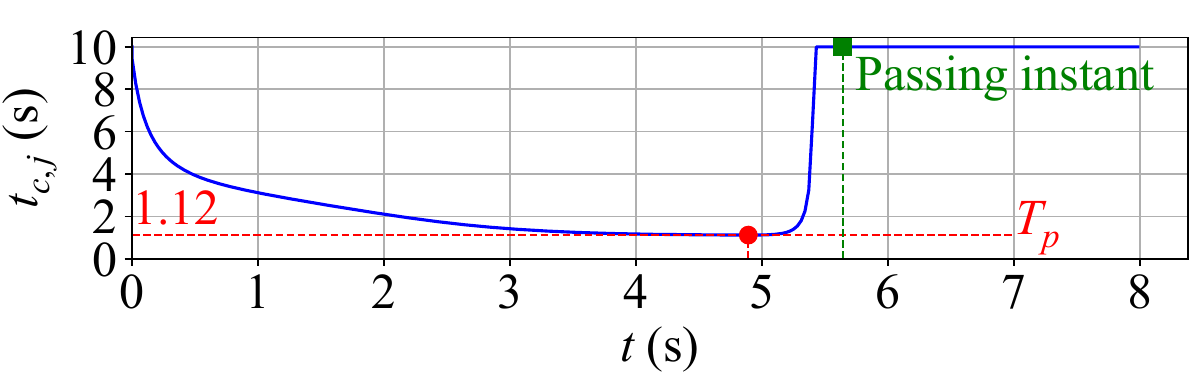}}
\subfigure[]{\includegraphics[width=0.48\textwidth]{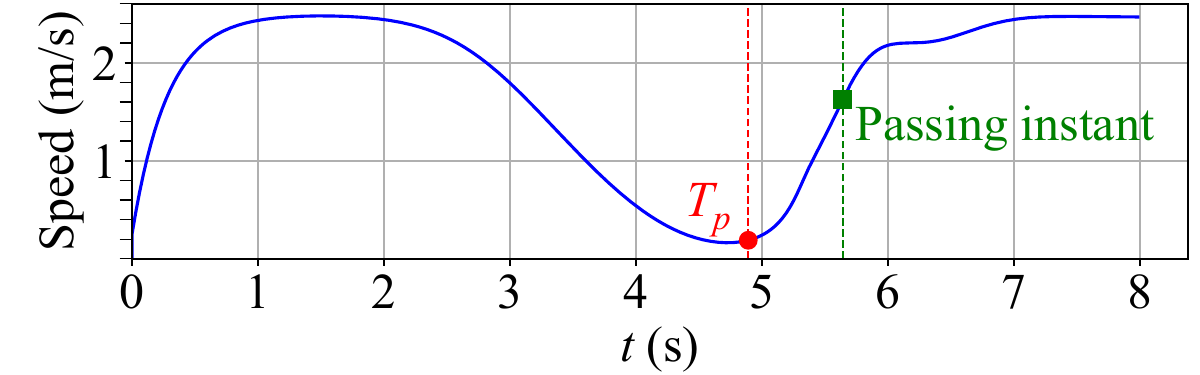}}
\subfigure[]{\includegraphics[width=0.48\textwidth]{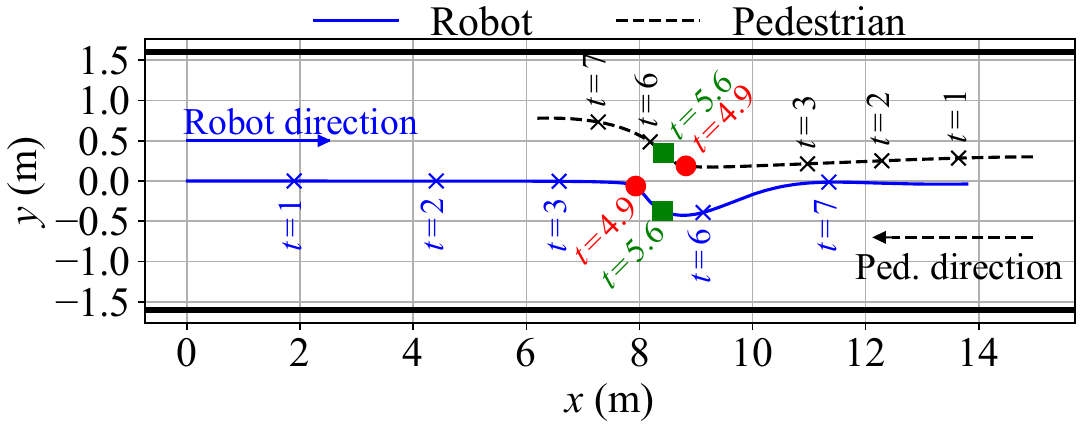}}
\caption{A sample TSFM simulation on a sidewalk with $3.2~m$ width.
(a) PTTC, denoted by $t_{c,j}$, changes as the robot approaches the pedestrian. 
$T_p$ is the minimum of $t_{c,j}$ and it happens just before the passing point.
$t_{c,j}\geq10$ and $t_{c,j}<0$ are set to $10$.
(b) The robot's speed changes as it approaches the pedestrian. 
In TSFM, $T_p$ and the minimum speed happen almost simultaneously if the social force is the dominant term.
(c) The TSFM NMR and an SFM pedestrian pass and avoid each other. 
The passing instances, timestamps, and $T_p$ are marked on the trajectories by \textcolor{darkgreen}{\rule{1.2ex}{1.2ex}}, \textcolor{black}{$\times$}, and \textcolor{black}{\scalebox{1.6}{$\bullet$}}, respectively.}
\label{fig:TSFMsimu}
\end{figure}
Fig.~\ref{fig:TSFMsimu} describes a sample TSFM trial on a sidewalk with $3.2~m$ width.
In the simulated interaction, the robot and the pedestrian move toward each other from opposite ends of the hall.
When they are far enough, the $t_{c,j}$ is large and $\|\vec{f}_{soc,j}\|\approx0$; We set $t_{c,j}\geq10~s$ to $t_{c,j}=10~s$ for convenience.\par
As they get closer, $t_{c,j}$ drops, Fig.~\ref{fig:TSFMsimu}(a), resulting in $\|\vec{f}_{soc,j}\|$ increase. 
Therefore, the avoidance maneuver starts.
According to \eqref{eq:v_des}--\eqref{eq:S}, the increase in $\|\vec{f}_{soc,j}\|$ results in a drop in desired speed $\|\vec{v}_{des}\|$ and the robot speed $\|\vec{v}\|$, Fig.~\ref{fig:TSFMsimu}(b), consequently.\par
Moments before the robot passes the pedestrian, $t=4.9~s$, $t_{c,j}$ bounces back from its minimum, $T_p$, and increases to infinity as the ``$\cos$" term in PTTC's denominator, from the dot product in~\eqref{eq:tcj}, converges to zero.
Finally, when the robot passes the pedestrian, the $t_{c,j}\leq0~s$, $\|\vec{f}_{soc,j}\|$ drops to zero, and the robot speeds up, as shown in Fig.~\ref{fig:TSFMsimu}(c).\par
\section{Sample trials}\label{sec:app:Trial}
\begin{figure*}[t]
\centering
\subfigure[]{\includegraphics[width=0.48\textwidth]{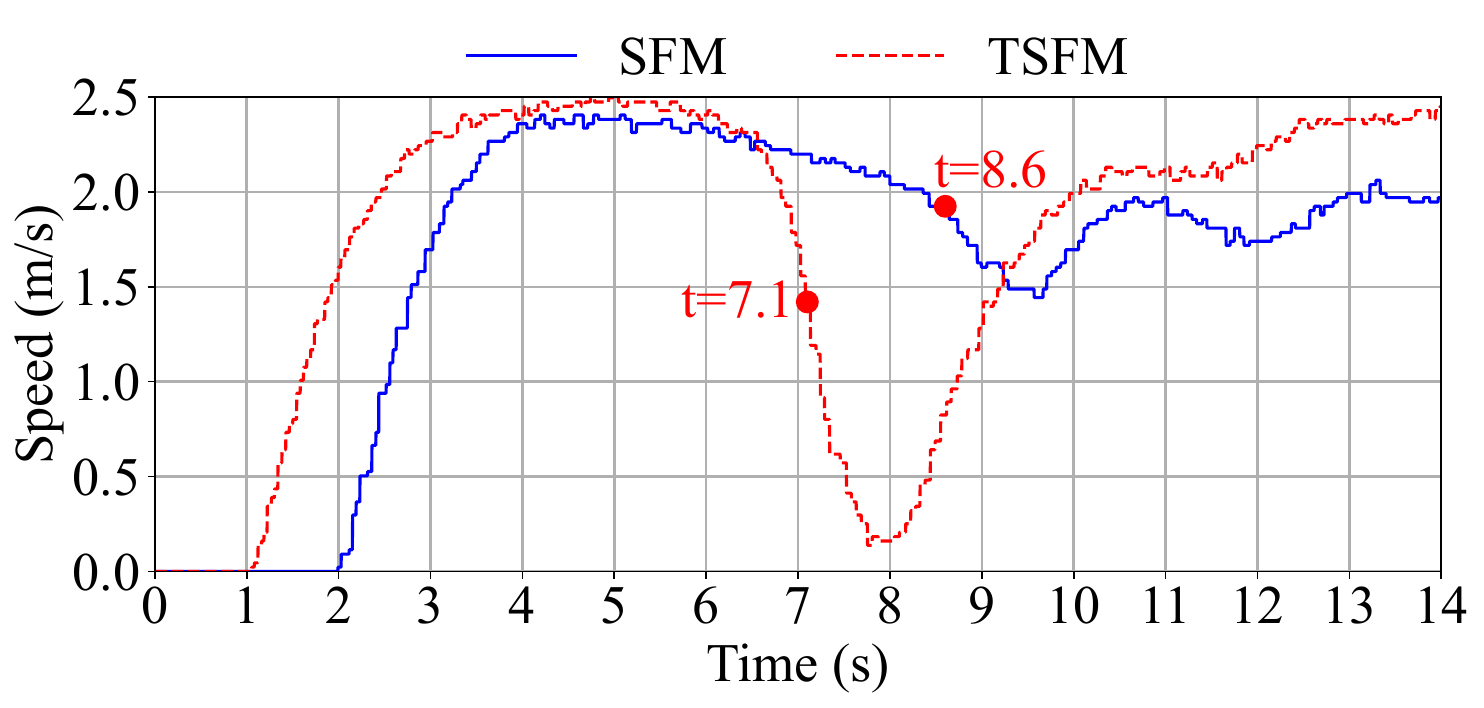}}
\subfigure[]{\includegraphics[width=0.48\textwidth]{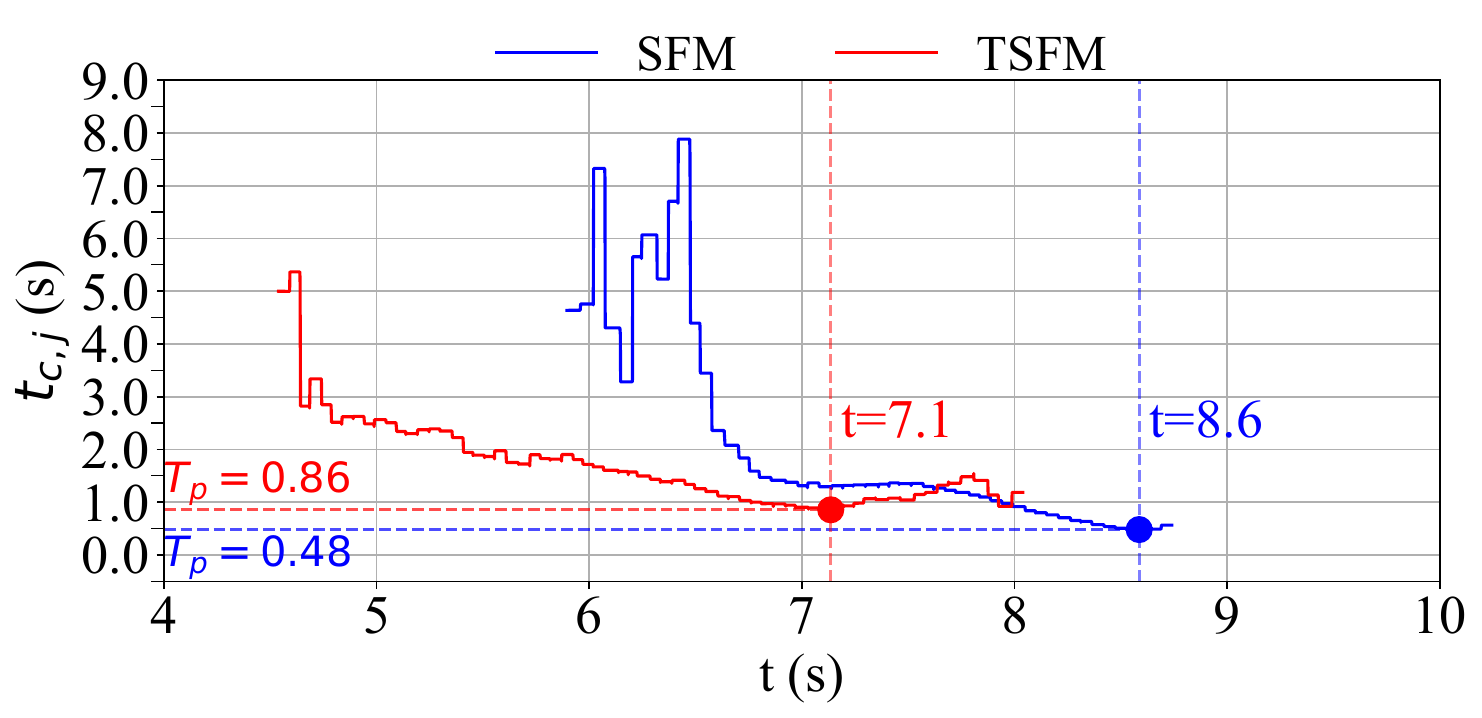}}\\
\subfigure[]{\includegraphics[width=0.48\textwidth]{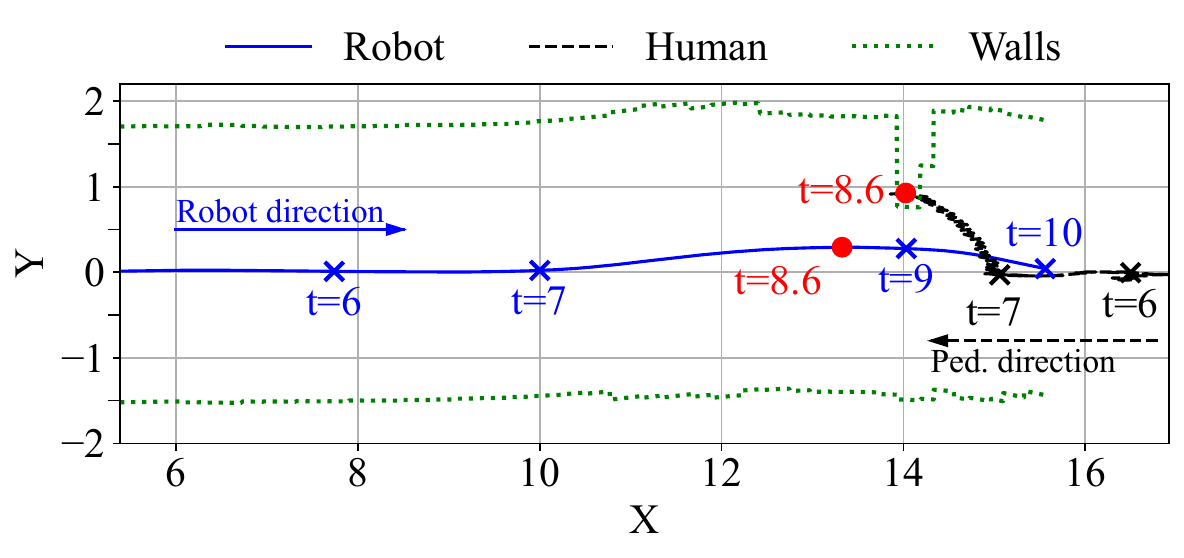}}
\subfigure[]{\includegraphics[width=0.48\textwidth]{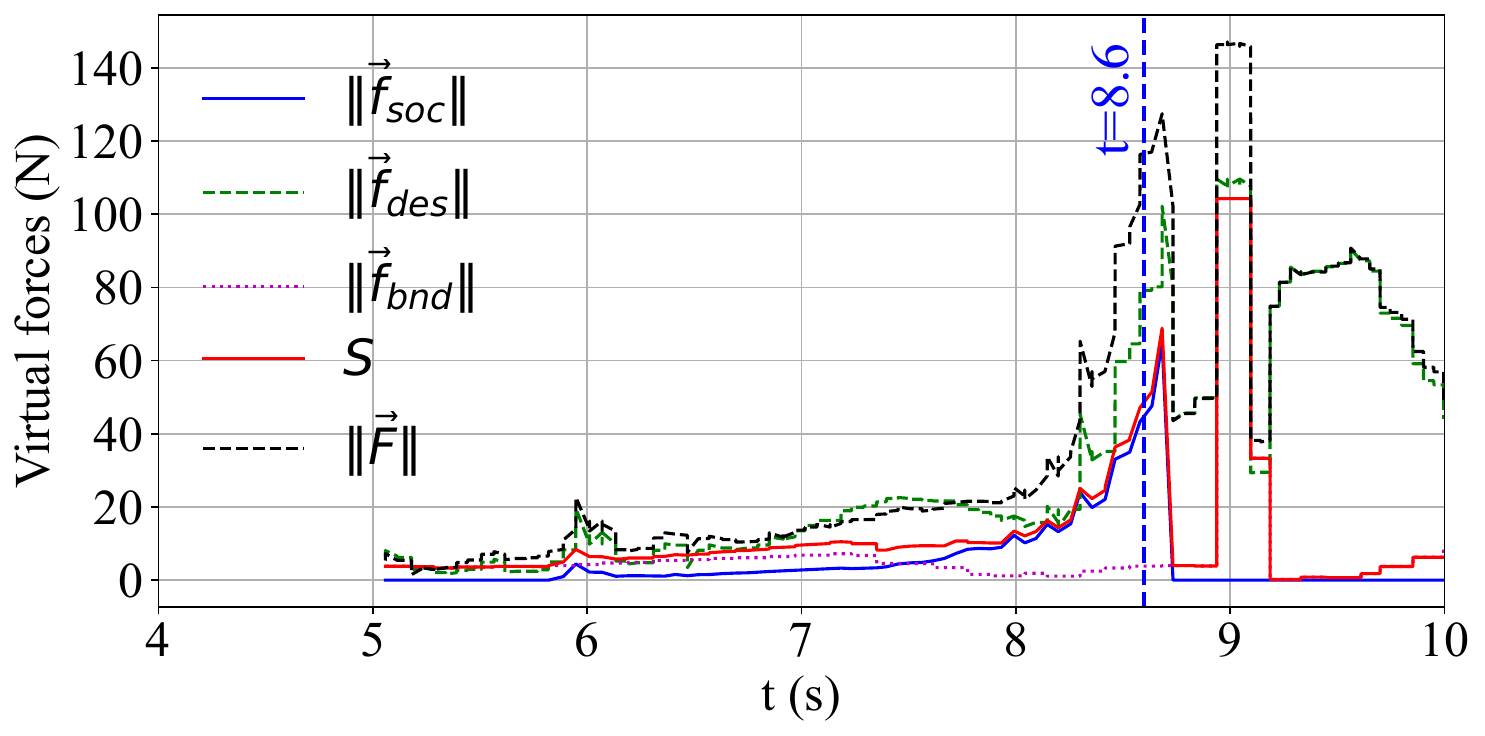}}\\
\subfigure[]{\includegraphics[width=0.48\textwidth]{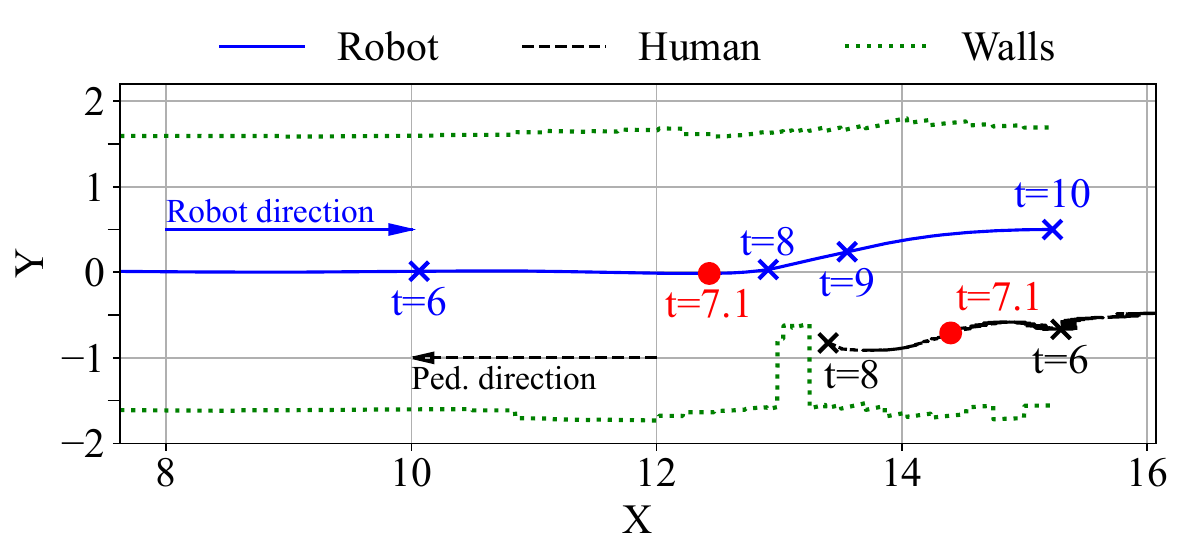}}
\subfigure[]{\includegraphics[width=0.48\textwidth]{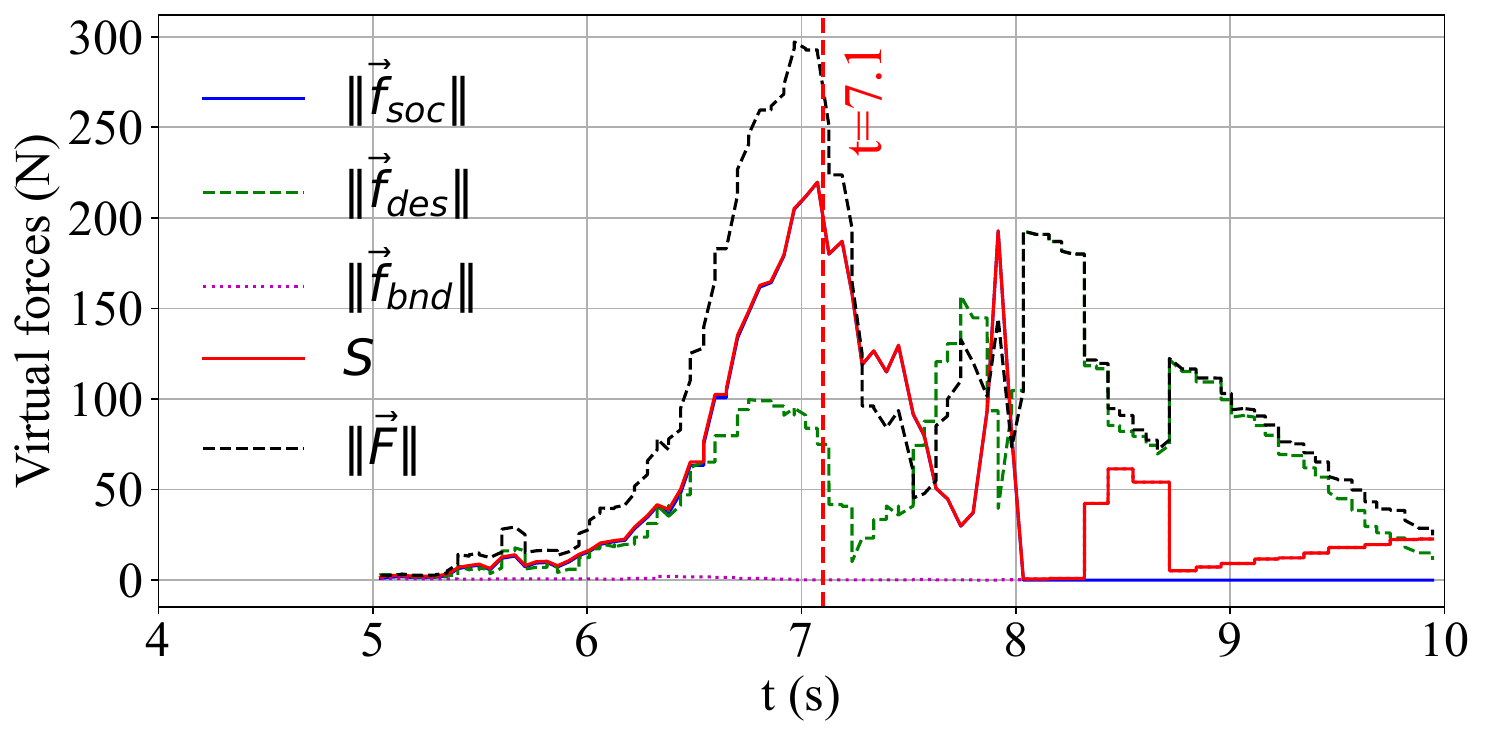}}
\caption{An SFM and a TSFM sample experimental trials with marked $T_p$s happened at $t=8.6~s$ and $t=7.1~s$, respectively. 
In addition, arbitrary timestamps are marked on the trajectories.
(a) The robot's speed variation during the trials. 
(b) The PTTC $t_{c,j}$ evolution during the trials. The minimum determines $T_p$ and its corresponding instance.
(c) An SFM NMR and a pedestrian pass and avoid each other.
The pedestrian moves from the right side of the robot to the left during the trial.
The robot moves to the left first and then swerves to the right to avoid a collision.
(d) The virtual forces vary during the SFM trial.  
(e) A TSFM NMR and a pedestrian pass and avoid each other.
(d) The virtual forces vary during the TSFM trial. \\
Fig. (a) and (b) include independent trials. SFM and TSFM plots are in the same figure to save space.
The abnormal change in the wall distance to the robot is because the LiDAR detects the pedestrian as a boundary when the robot and the walker are side by side.}
\label{fig:expTrialTr}
\end{figure*}
This section presents two experimental trials that employ SFM and TSFM, and compares their performance.
During the trials, the robot moves from the left side of the hallway to the right side, and the pedestrian moves in the opposite direction.
Fig.~\ref{fig:expTrialTr}(a) shows the robot speed for the trials.
During free motion, the robot's speed converges to the set value of $v_{max}$, i.e., $2.5~m/s$.
\par
Around $t=5~s$ for both trials, the robot detects the human.
The pedestrian disappears from the camera around $t=8.7~s$ and $t=8.0~s$ in SFM and TSFM trials, respectively.
Fig.~\ref{fig:expTrialTr}(c) and \ref{fig:expTrialTr}(e) show the agents' trajectories from the human appearance up to the last detection moment.
In addition, they depict the detected walls during the trial.
\par
Fig.~\ref{fig:expTrialTr}(b) shows the PTTC evolution in both trials. 
During the first moments of the pedestrian's appearance, the measurements vary significantly, and therefore, PTTC is not accurate.
However, as the agents get closer, the PTTC becomes more persistent. 
The figure also marks the minimum of PTTC, $T_p$, during the trials.
\par
Fig.~\ref{fig:expTrialTr}(d) and Fig.~\ref{fig:expTrialTr}(f) show $\|\vec{f}_{soc,j}\|$, $\|\vec{f}_{bnd,k}\|$, and $\|\vec{f}_{des}\|$ during the sample SFM and TSFM trials, respectively.
In addition, they show $S$ defined in \eqref{eq:S} and  $\|\vec{F}\|$.
Note that $\vec{f}_{des}$ affects $\|\vec{F}\|$ but not $S$.
In addition, $S$ is the summation of the norms where $\|\vec{F}\|$ sums the forces up as vectors and then applies the norm.
\par
Initially, the pedestrian is out of range of the onboard depth camera, and the robot cannot see the pedestrian, $\|\vec{f}_{soc,j}\|=0$.
In addition, the robot remains at the center of the hall, $\|\vec{f}_{bnd,k}\|\approx0$.
Therefore, $S\approx0$ and the $v_{des}$ is very close to $v_{max}$.
Consequently, $\vec{f}_{des}$ is the only effective force and the robot speeds up to $v_{max}$ freely, regardless of the model, Fig.~\ref{fig:expTrialTr}(a).
\par
The robot detects the pedestrian at $t=5.9~s$ and $t=4.5~s$ in SFM and TSFM trials, respectively.
For a short time after the initial detection, the detected relative distance and PTTC are inaccurate and vary significantly; see Fig.~\ref{fig:expTrialTr}(b).
Nevertheless, the inaccuracy doesn't affect the robot's behavior because the distance is too large to cause any significant forces, considering the exponential nature of the terms.
As the robot and the pedestrian close in and the detection stabilizes, $\vec{f}_{soc,j}$ gradually rises and pushes the robot away from the pedestrian path.
\par
In Fig.~\ref{fig:expTrialTr}(c), SFM trial, the robot deviates from its initial path at about $t=7~s$.
Since the pedestrian is slightly to the right side of the robot in the hall, the robot moves to the left.
However, around the same time, the pedestrian, not seeing any action from the robot up to that moment, decides to suddenly change its course to the left side of the robot to avoid it.
As a result, at $t=8~s$ the robot and the pedestrian are on the left side of the hall getting closer to each other.
Thus, Fig.~\ref{fig:expTrialTr}(d) indicates sudden peaks in force magnitudes between $t=8~s$ and $t=9~s$, resulting in a drop in the robot's speed seen in Fig.~\ref{fig:expTrialTr}(a) during the span.
Finally, the robot passes the pedestrian and accelerates to $v_{max}$ while moving to the hallway centerline.
\par
In Fig.~\ref{fig:expTrialTr}(e), TSFM trial, the robot deviates from its path at about $t=8~s$.
However, Fig.~\ref{fig:expTrialTr}(a) shows that it already significantly slowed down at about $t=6~s$ and reached $0.2~m/s$ at $t=7.8~s$.
Fig.~\ref{fig:expTrialTr}(f) indicates that the critical moment happened at about $t=7.1~s$ where the forces' magnitude maximized.
Still, the robot deviates from its trajectory and avoids the human moving at low speed, and then speeds up to $v_{max}$.
\par
Fig.~\ref{fig:expTrialTr}(c) and Fig.~\ref{fig:expTrialTr}(e) contain abnormalities regarding the left and right walls when the robot is passing the pedestrian, respectively.
The reason for the abnormalities is the misidentification of the pedestrian as a boundary by LiDAR when the agents are side by side.
Since we focus on the robot's avoidance behavior, and at this point, the robot has practically passed the pedestrian, the abnormality doesn't interfere with the results.
However, the misinterpretation changes the robot's behavior in free motion after passing the pedestrian and must be addressed in applications.
\par
Fig.~\ref{fig:expTrialTr}(b) is the main takeaway from the section.
It demonstrates the evolution of PTTC during the trials. 
Ignoring the initial fluctuations due to measurement instability, the PTTC decreases as the agents approach and increases just before they pass each other, creating a local minimum point $T_p$.
Following \cite{JAFARI2024-2-NATCOM} for e-scooters, we assume that pedestrian comfort correlates with $T_p$ for mobile robots, too.
During the trials, we collected $T_p$ as the minimum of $t_{c,j}$ defined in \eqref{eq:tcj}, and the speed in the vicinity of the pedestrian using \eqref{eq:v_vic}.\par
\section{Statistics}\label{sec:app:statistics}
\begin{table}[hb]
\centering
\renewcommand\arraystretch{1.05}
\caption{Box plot details for Fig.~\ref{fig:boxes} and Fig.~\ref{fig:Survey_boxes}. The table treats each trial as an independent sample since it describes the box/violin plots.}
\label{tab:Box_det_Fig5and6}
\begin{tabular}{||c|c|c|c|c||}
\hline
 & \textbf{TSFM} & \textbf{R14} & \textbf{SFM} & \textbf{R28} \\
\hline\hline
\multicolumn{5}{||c||}{Box plot details for Fig.~\ref{fig:boxes}(a)} \\
\hline
\textbf{N / Median} & 160 / 1.00 & 79 / 0.69 & 159 / 0.56 & 79 / 0.48 \\
\hline
\textbf{Mean / SD} & 1.01 / 0.30 & 0.71 / 0.23 & 0.58 / 0.16 & 0.53 / 0.17 \\
\hline
\textbf{Q3 / Q1} & 1.21 / 0.79 & 0.82 / 0.54 & 0.66 / 0.47 & 0.60 / 0.42 \\
\hline
\textbf{U/L notch} & 1.05 / 0.95 & 0.74 / 0.64 & 0.58 / 0.53 & 0.51 / 0.45 \\
\hline
\textbf{U/L Whis.} & 1.77 / 0.39 & 1.24 / 0.26 & 0.92 / 0.23 & 0.86 / 0.25 \\
\hline
\multicolumn{5}{||c||}{Box plot details for Fig.~\ref{fig:boxes}(b)} \\
\hline
\textbf{N / Median} & 160 / 0.70 & 79 / 1.42 & 159 / 1.92 & 79 / 2.85 \\
\hline
\textbf{Mean / SD} & 0.72 / 0.15 & 1.42 / 0.01 & 1.85 / 0.23 & 2.85 / 0.01 \\
\hline
\textbf{Q3 / Q1} & 0.80 / 0.62 & 1.43 / 1.42 & 1.98 / 1.79 & 2.85 / 2.85 \\
\hline
\textbf{U/L notch} & 0.72 / 0.68 & 1.42 / 1.42 & 1.94 / 1.89 & 2.85 / 2.85 \\
\hline
\textbf{U/L Whis.} & 1.05 / 0.42 & 1.43 / 1.41 & 2.13 / 1.51 & 2.86 / 2.84 \\
\hline
\multicolumn{5}{||c||}{Box plot details for Fig.~\ref{fig:boxes}(c)} \\
\hline
\textbf{N / Median} & 160 / 0.63 & 79 / 0.50 & 159 / 0.43 & 79 / 0.38 \\
\hline
\textbf{Mean / SD} & 0.62 / 0.11 & 0.50 / 0.11 & 0.43 / 0.09 & 0.40 / 0.10 \\
\hline
\textbf{Q3 / Q1} & 0.70 / 0.55 & 0.56 / 0.42 & 0.48 / 0.38 & 0.45 / 0.34 \\
\hline
\textbf{U/L notch} & 0.65 / 0.61 & 0.53 / 0.47 & 0.44 / 0.41 & 0.40 / 0.36 \\
\hline
\textbf{U/L Whis.} & 0.83 / 0.32 & 0.76 / 0.23 & 0.63 / 0.24 & 0.61 / 0.22 \\
\hline
\multicolumn{5}{||c||}{Box plot details for Fig.~\ref{fig:boxes}(d)} \\
\hline
\textbf{N / Median} & 160 / 0.28 & 79 / 0.57 & 159 / 0.77 & 79 / 1.14 \\
\hline
\textbf{Mean / SD} & 0.29 / 0.06 & 0.57 / 0.00 & 0.74 / 0.09 & 1.14 / 0.00 \\
\hline
\textbf{Q3 / Q1} & 0.32 / 0.25 & 0.57 / 0.57 & 0.79 / 0.72 & 1.14 / 1.14 \\
\hline
\textbf{U/L notch} & 0.29 / 0.27 & 0.57 / 0.57 & 0.78 / 0.76 & 1.14 / 1.14 \\
\hline
\textbf{U/L Whis.} & 0.42 / 0.17 & 0.57 / 0.57 & 0.85 / 0.60 & 1.14 / 1.14 \\
\hline
\multicolumn{5}{||c||}{Box plot details for Fig.~\ref{fig:boxes}(e)} \\
\hline
\textbf{N / Median} & 160 / 0.91 & 79 / 0.96 & 159 / 0.69 & 79 / 0.99 \\
\hline
\textbf{Mean / SD} & 0.99 / 0.35 & 0.95 / 0.29 & 0.73 / 0.23 & 1.02 / 0.31 \\
\hline
\textbf{Q3 / Q1} & 1.22 / 0.72 & 1.12 / 0.74 & 0.85 / 0.56 & 1.25 / 0.81 \\
\hline
\textbf{U/L notch} & 0.97 / 0.85 & 1.03 / 0.90 & 0.72 / 0.65 & 1.07 / 0.91 \\
\hline
\textbf{U/L Whis.} & 1.95 / 0.39 & 1.50 / 0.32 & 1.21 / 0.28 & 1.80 / 0.47 \\
\hline
\multicolumn{5}{||c||}{Box plot details for Fig.~\ref{fig:Survey_boxes}(a)} \\
\hline
\textbf{N / Median} & 160 / 4 & 79 / 4 & 160 / 3 & 80 / 4 \\
\hline
\textbf{Mean / SD} & 3.52 / 1.07 & 4.04 / 0.94 & 2.97 / 1.10 & 3.54 / 1.12 \\
\hline
\textbf{Q3 / Q1} & 4 / 3 & 5 / 3.5 & 4 / 2 & 4 / 3 \\
\hline
\textbf{U/L notch} & 4.12 / 3.88 & 4.26 / 3.74 & 3.25 / 2.75 & 4.18 / 3.82 \\
\hline
\textbf{U/L Whis.} & 5 / 2 & 5 / 2 & 5 / 1 & 5 / 2 \\
\hline
\multicolumn{5}{||c||}{Box plot details for Fig.~\ref{fig:Survey_boxes}(b)} \\
\hline
\textbf{N / Median} & 160 / 4 & 79 / 5 & 160 / 4 & 80 / 4 \\
\hline
\textbf{Mean / SD} & 3.51 / 1.15 & 4.43 / 0.69 & 3.92 / 0.84 & 4.31 / 0.74 \\
\hline
\textbf{Q3 / Q1} & 4 / 3 & 5 / 4 & 5 / 3 & 5 / 4 \\
\hline
\textbf{U/L notch} & 4.12 / 3.88 & 5.18 / 4.82 & 4.25 / 3.75 & 4.18 / 3.82 \\
\hline
\textbf{U/L Whis.} & 5 / 2 & 5 / 3 & 5 / 2 & 5 / 3 \\
\hline
\multicolumn{5}{||c||}{Box plot details for Fig.~\ref{fig:Survey_boxes}(c)} \\
\hline
\textbf{N / Median} & 160 / 4 & 79 / 4 & 160 / 3 & 80 / 4 \\
\hline
\textbf{Mean / SD} & 3.66 / 1.09 & 3.81 / 0.93 & 2.86 / 1.23 & 3.55 / 1.11 \\
\hline
\textbf{Q3 / Q1} & 5 / 3 & 4.5 / 3 & 4 / 2 & 4 / 3 \\
\hline
\textbf{U/L notch} & 4.25 / 3.75 & 4.26 / 3.74 & 3.25 / 2.75 & 4.18 / 3.82 \\
\hline
\textbf{U/L Whis.} & 5 / 1 & 5 / 1 & 5 / 1 & 5 / 2 \\
\hline
\multicolumn{5}{||c||}{Box plot details for Fig.~\ref{fig:Survey_boxes}(d)} \\
\hline
\textbf{N / Median} & 160 / 4 & 79 / 4 & 160 / 3 & 80 / 3 \\
\hline
\textbf{Mean / SD} & 3.80 / 0.88 & 4.28 / 0.64 & 3.19 / 0.97 & 3.17 / 1.32 \\
\hline
\textbf{Q3 / Q1} & 4 / 3 & 5 / 4 & 4 / 3 & 4 / 2 \\
\hline
\textbf{U/L notch} & 4.12 / 3.88 & 4.18 / 3.82 & 3.12 / 2.88 & 3.35 / 2.65 \\
\hline
\textbf{U/L Whis.} & 5 / 2 & 5 / 3 & 5 / 2 & 5 / 1 \\
\hline
\end{tabular}
\end{table}
This section contains the statistical analysis results details. 
Table~\textcolor{black}{\ref{tab:Box_det_Fig5and6}} contains the details presented in Fig.~\ref{fig:boxes} and in Fig.~\ref{fig:Survey_boxes}.
The trials are treated as independent samples to match the descriptive box/violin plots.
Tables~\ref{tab:St_test_Fig5a} to~\ref{tab:St_test_Fig5e} contain statistical comparison test results, including p-values and the confidence intervals C.I. for each pair of groups per Fig.~\ref{fig:boxes}(a)--(e).
Similarly, Tables~\ref{tab:St_test_Fig6a} to~\ref{tab:St_test_Fig6d} contain statistical comparison test results, including p-values and the confidence intervals C.I. for each pair of groups per Fig.~\ref{fig:Survey_boxes}(a)--(d).
There are 32 participants in total.
The following tests are performed using the following design.
\begin{itemize}
    \item TSFM vs SFM: paired t-test (32 participants overlap)
    \item TSFM vs R14: paired t-test (16 participants overlap)
    \item TSFM vs R28: paired t-test (16 participants overlap)
    \item SFM vs R14: paired t-test (16 participants overlap)
    \item SFM vs R28: paired t-test (16 participants overlap)
    \item R14 vs R28: not paired (No participants overlap); we use an independent-samples test (i.e., Welch).
\end{itemize}
\begin{table}[H]
\centering
\renewcommand\arraystretch{1.1}
\setlength{\tabcolsep}{3pt}
\caption{
Pairwise results for Fig.~\ref{fig:boxes}(a): we report p-value and 95$\%$ CI of mean differences.
R14–R28 test uses Welch’s independent-samples t-test; all other comparisons use paired-samples t-tests on participant-level means (n = 32 for TSFM–SFM; n = 16 for pairs involving R14 or R28).
p-values smaller than 0.005 are shown by $0^+$. 
The table is symmetric.}
\label{tab:St_test_Fig5a}
\begin{tabular}{||c|c|c|c||}
\hline
 & \textbf{TSFM} & \textbf{R14} & \textbf{SFM}\\
\hline\hline
\textbf{R14} & $0^+$, [-0.34, -0.12] & -- & -- \\
\hline
\textbf{SFM} & $0^+$, [-0.50, -0.36] & $0^+$, [-0.24, -0.05] & -- \\
\hline
\textbf{R28} & $0^+$, [-0.67, -0.44] & $0^+$, [-0.29, -0.08] & $0.12$, [-0.14, 0.02] \\
\hline
\end{tabular}%
\end{table}
\begin{table}[H]
\centering
\renewcommand\arraystretch{1.1}
\setlength{\tabcolsep}{5.5pt}
\caption{Statistical test, Fig.~\ref{fig:boxes}(b); details as in Table~\ref{tab:St_test_Fig5a}.}
\label{tab:St_test_Fig5b}
\begin{tabular}{||c|c|c|c||}
\hline
 & \textbf{TSFM} & \textbf{R14} & \textbf{SFM}\\
\hline\hline
\textbf{R14} & $0^+$, [0.67, 0.75] & -- & -- \\
\hline
\textbf{SFM} & $0^+$, [1.06, 1.21] & $0^+$, [0.26, 0.50] & -- \\
\hline
\textbf{R28} & $0^+$, [2.06, 2.19] & $0^+$, [1.43, 1.43] & $0^+$, [0.88, 1.00] \\
\hline
\end{tabular}
\end{table}
\begin{table}[H]
\centering
\renewcommand\arraystretch{1.1}
\setlength{\tabcolsep}{3pt}
\caption{Statistical test, Fig.~\ref{fig:boxes}(c); details as in Table~\ref{tab:St_test_Fig5a}.}
\label{tab:St_test_Fig5c}
\begin{tabular}{||c|c|c|c||}
\hline
 & \textbf{TSFM} & \textbf{R14} & \textbf{SFM}\\
\hline\hline
\textbf{R14} & $0^+$, [-0.14, -0.05] & -- & -- \\
\hline
\textbf{SFM} & $0^+$, [-0.21, -0.16] & $0^+$, [-0.12, -0.03] & -- \\
\hline
\textbf{R28} & $0^+$, [-0.29, -0.20] & $0^+$, [-0.15, -0.04] & $0.11$, [-0.08, 0.01] \\
\hline
\end{tabular}
\end{table}
\begin{table}[H]
\centering
\renewcommand\arraystretch{1.1}
\setlength{\tabcolsep}{5.5pt}
\caption{Statistical test, Fig.~\ref{fig:boxes}(d); details as in Table~\ref{tab:St_test_Fig5a}.}
\label{tab:St_test_Fig5d}
\begin{tabular}{||c|c|c|c||}
\hline
 & \textbf{TSFM} & \textbf{R14} & \textbf{SFM}\\
\hline\hline
\textbf{R14} & $0^+$, [0.27, 0.30] & -- & -- \\
\hline
\textbf{SFM} & $0^+$, [0.43, 0.48] & $0^+$, [0.10, 0.20] & -- \\
\hline
\textbf{R28} & $0^+$, [0.82, 0.87] & $0^+$, [0.57, 0.57] & $0^+$, [0.35, 0.40] \\
\hline
\end{tabular}
\end{table}
\begin{table}[H]
\centering
\renewcommand\arraystretch{1.1}
\setlength{\tabcolsep}{4pt}
\caption{Statistical test, Fig.~\ref{fig:boxes}(e); details as in Table~\ref{tab:St_test_Fig5a}.}
\label{tab:St_test_Fig5e}
\begin{tabular}{||c|c|c|c||}
\hline
 & \textbf{TSFM} & \textbf{R14} & \textbf{SFM}\\
\hline\hline
\textbf{R14} & $0.59$, [-0.10, 0.17] & -- & -- \\
\hline
\textbf{SFM} & $0^+$, [-0.34, -0.19] & $0^+$, [-0.35, -0.18] & -- \\
\hline
\textbf{R28} & $0.38$, [-0.16, 0.06] & $0.33$, [-0.08, 0.23] & $0^+$, [0.14, 0.37] \\
\hline
\end{tabular}
\end{table}
\begin{table}[H]
\centering
\renewcommand\arraystretch{1.1}
\setlength{\tabcolsep}{4pt}
\caption{Pairwise results for Fig.~\ref{fig:Survey_boxes}(a): we report p-value and  95$\%$ CI of the effect size (rank-biserial correlation).
R14--R28 test uses Mann--Whitney U test; all other comparisons use Wilcoxon signed-rank tests on participant-level medians (n = 32 for TSFM--SFM; n = 16 for pairs involving R14 or R28).
p-values smaller than 0.005 are shown by $0^+$. 
The table is symmetric.}
\label{tab:St_test_Fig6a}
\begin{tabular}{||c|c|c|c||}
\hline
 & \textbf{TSFM} & \textbf{R14} & \textbf{SFM}\\
\hline\hline
\textbf{R14} & $0.05$, [0.33, 1.00] & -- & -- \\
\hline
\textbf{SFM} & $0^{+}$, [-1.00, -0.62] & $0^{+}$, [-1.00, -1.00] & -- \\
\hline
\textbf{R28} & $0.87$, [-0.40, 0.69] & $0.13$, [-0.63, 0.09] & $0.07$, [-0.09, 1.00] \\
\hline
\end{tabular}
\end{table}
\begin{table}[H]
\centering
\renewcommand\arraystretch{1.1}
\setlength{\tabcolsep}{4.5pt}
\caption{Statistical test, Fig.~\ref{fig:Survey_boxes}(b); details as in Table~\ref{tab:St_test_Fig6a}.}
\label{tab:St_test_Fig6b}
\begin{tabular}{||c|c|c|c||}
\hline
 & \textbf{TSFM} & \textbf{R14} & \textbf{SFM}\\
\hline\hline
\textbf{R14} & $0^{+}$, [1.00, 1.00] & -- & -- \\
\hline
\textbf{SFM} & $0.02$, [0.07, 0.89] & $0.01$, [-1.00, -1.00] & -- \\
\hline
\textbf{R28} & $0^{+}$, [1.00, 1.00] & $0.76$, [-0.39, 0.30] & $0.01$, [1.00, 1.00] \\
\hline
\end{tabular}
\end{table}
\begin{table}[H]
\centering
\renewcommand\arraystretch{1.1}
\setlength{\tabcolsep}{4pt}
\caption{Statistical test, Fig.~\ref{fig:Survey_boxes}(c); details as in Table~\ref{tab:St_test_Fig6a}.}
\label{tab:St_test_Fig6c}
\begin{tabular}{||c|c|c|c||}
\hline
 & \textbf{TSFM} & \textbf{R14} & \textbf{SFM}\\
\hline\hline
\textbf{R14} & $0.30$, [-0.20, 0.85] & -- & -- \\
\hline
\textbf{SFM} & $0^{+}$, [-0.93, -0.41] & $0.01$, [-1.00, -0.29] & -- \\
\hline
\textbf{R28} & $0.80$, [-0.40, 0.69] & $0.62$, [-0.48, 0.28] & $0.01$, [0.50, 1.00] \\
\hline
\end{tabular}
\end{table}
\begin{table}[H]
\centering
\renewcommand\arraystretch{1.1}
\setlength{\tabcolsep}{3.5pt}
\caption{Statistical test, Fig.~\ref{fig:Survey_boxes}(d); details as in Table~\ref{tab:St_test_Fig6a}.}
\label{tab:St_test_Fig6d}
\begin{tabular}{||c|c|c|c||}
\hline
 & \textbf{TSFM} & \textbf{R14} & \textbf{SFM}\\
\hline\hline
\textbf{R14} & $0.01$, [0.38, 1.00] & -- & -- \\
\hline
\textbf{SFM} & $0^{+}$, [-1.00, -0.44] & $0^{+}$, [-1.00, -1.00] & -- \\
\hline
\textbf{R28} & $0.01$, [-1.00, -0.27] & $0^{+}$, [-0.84, -0.23] & $0.22$, [-0.78, 0.43] \\
\hline
\end{tabular}
\end{table}
\end{document}